\documentclass[nohyperref]{article}

\usepackage{microtype}
\usepackage{graphicx}
\usepackage{subfigure}
\usepackage{booktabs} 

\usepackage{hyperref}

\usepackage[accepted]{icml2022}

\usepackage{amsmath}
\usepackage{amssymb}
\usepackage{mathtools}
\usepackage{amsthm}
\usepackage{ulem} 

\theoremstyle{plain}
\newtheorem{theorem}{Theorem}[section]
\newtheorem{proposition}[theorem]{Proposition}
\newtheorem{lemma}[theorem]{Lemma}
\newtheorem{corollary}[theorem]{Corollary}
\theoremstyle{definition}
\newtheorem{definition}[theorem]{Definition}

\theoremstyle{remark}

\usepackage[textsize=tiny]{todonotes}

\usepackage{amsmath,amsfonts,bm}

\def\dd{\mathrm{d}}

\def\eqref#1{equation~\ref{#1}}

\def\1{\bm{1}}

\DeclareMathAlphabet{\mathsfit}{\encodingdefault}{\sfdefault}{m}{sl}
\SetMathAlphabet{\mathsfit}{bold}{\encodingdefault}{\sfdefault}{bx}{n}

\def\gG{{\mathcal{G}}}

\def\gS{{\mathcal{S}}}

\def\sA{{\mathbb{A}}}
\def\sB{{\mathbb{B}}}

\def\sI{{\mathbb{I}}}
\def\sJ{{\mathbb{J}}}

\def\sR{{\mathbb{R}}}

\def\sV{{\mathbb{V}}}

\icmltitlerunning{How Powerful are Spectral Graph Neural Networks}

\begin{document}

\twocolumn[
\icmltitle{How Powerful are Spectral Graph Neural Networks}



\icmlsetsymbol{equal}{*}

\begin{icmlauthorlist}
\icmlauthor{Xiyuan Wang}{pku}
\icmlauthor{Muhan Zhang}{pku,bigai}
\end{icmlauthorlist}

\icmlaffiliation{pku}{Institute for Artificial Intelligence, Peking University}
\icmlaffiliation{bigai}{Beijing Institute for General Artificial Intelligence}

\icmlcorrespondingauthor{Muhan Zhang}{muhan@pku.edu.cn}

\icmlkeywords{Expressive Power, Graph Neural Network}

\vskip 0.3in
]



\printAffiliationsAndNotice{}

\begin{abstract}
Spectral Graph Neural Network is a kind of Graph Neural Network (GNN) based on graph signal filters. Some models able to learn arbitrary spectral filters have emerged recently. However, few works analyze the expressive power of spectral GNNs. This paper studies spectral GNNs’ expressive power theoretically. We first prove that even spectral GNNs without nonlinearity can produce arbitrary graph signals and give two conditions for reaching universality. They are: 1) no multiple eigenvalues of graph Laplacian, and 2) no missing frequency components in node features. We also establish a connection between the expressive power of spectral GNNs and Graph Isomorphism (GI) testing, the latter of which is often used to characterize spatial GNNs’ expressive power. Moreover, we study the difference in empirical performance among different spectral GNNs with the same expressive power from an optimization perspective, and motivate the use of an orthogonal basis whose weight function corresponds to the graph signal density in the spectrum. Inspired by the analysis, we propose JacobiConv, which uses Jacobi basis due to its orthogonality and flexibility to adapt to a wide range of weight functions. JacobiConv deserts nonlinearity while outperforming all baselines on both synthetic and real-world datasets.
\end{abstract}

\section{Introduction}
Graph Neural Networks (GNNs) have achieved state-of-the-art performance on almost all tasks among various graph representation learning methods~\citep{intro_nlp,intro_bio,intro_soc}. Spectral GNNs are a kind of GNNs that design graph signal filters in the spectral domain. Though various models have emerged, spectral GNNs' expressive power is still under-researched. Moreover, these models differ mainly in the basis choices of the spectral filters; however, to our best knowledge, no study has systematically explained these differences and studied the advantages and disadvantages of different bases. 

Existing spectral GNNs can be summarized into a general form: first transforming the spatial signal $X$ through an MLP, then applying spectral filters parameterized by a polynomial of the normalized Laplacian $\hat{L}$, and finally applying another MLP to the filtered signal. By designing/learning the polynomial coefficients, spectral GNNs can simulate a wide range of filters (low-pass, band-pass, high-pass) in the spectral domain, enabling GNNs to work on not only homophilic but also heterophilic graphs~\citep{GPRGNN}. 

However, a natural question is whether the MLPs or nonlinearity are useful at all, or are spectral filters enough? To study this problem, we remove nonlinearity from spectral GNNs and explore the expressive power of such linear spectral GNNs whose power relies only on the spectral filters. We prove that linear GNNs are \textit{universal} under some mild conditions, i.e., they are powerful enough to produce arbitrary predictions without relying on MLPs. Our results show that nonlinearity is unnecessary for spectral GNNs to reach high expressiveness, which is also verified in our experiments. Moreover, we analyze spectral GNNs' universality conditions from a Graph Isomorphism (GI) testing perspective. The latter is often used to characterize spatial GNNs' expressive power~\citep{HowPowerfulAreGNNs}. Our results, for the first time, build a bridge between the expressivity analyses of spectral GNNs and spatial GNNs.

Next, we notice that spectral GNNs with different polynomial bases of the spectral filters have the same expressive power but different empirical performance. To study this difference, we analyze the optimization of such models. By checking the Hessian matrix of linear spectral GNNs near the global minimum, we find that using an orthogonal basis with the density of graph signal as the weight function can maximize the convergence speed. 

Inspired by these discussions, we propose a novel expressive spectral GNN, JacobiConv. JacobiConv deserts nonlinearity, approximates filter functions with Jacobi basis, and is flexible enough in weight function choices to adapt to a wide range of graph signal densities. We also design a novel Polynomial Coefficient Decomposition (PCD) technique to improve the filter coefficient optimization. In numerical experiments, we first test the expressive power of JacobiConv to approximate filter functions on synthetic datasets. JacobiConv achieves the lowest loss on learning the filter functions compared to state-of-the-art spectral GNNs. We also show that JacobiConv outperforms all baselines on ten real-world datasets by up to $12\%$. 

\section{Preliminaries}
For any matrix $M\in \sR^{a\times b}$, $M_{i}$ is the $i^{\text{th}}$ row of $M$, $M_{:i}$ is the $i^{\text{th}}$ column of $M$, $M_{\sA\sB}$ is the submatrix of $M$ corresponding to row index set $\sA$ and column index set $\sB$. Let $\delta_{ij}$ denote Kronecker delta: $1$ if $i=j$, and $0$ otherwise. We define \textit{condition number} of a matrix $M$ as $\kappa(M)=\frac{|\lambda_{\text{max}}|}{|\lambda_{\text{min}}|}$, where $\lambda_{\text{min}}$, $\lambda_{\text{max}}$ are the minimum and maximum eigenvalues of $M$, respectively. If $M$ is singular, $\kappa(M)=+\infty$.

Let $\gG = (\sV,\mathbb{E},X)$ denote an undirected graph with a finite node set $\sV=\{1,2,...,n\}$, an edge set $\mathbb{E}\subseteq \sV\times\sV$ and a node feature matrix $X\in \sR^{n\times d}$, whose $i^{\text{th}}$ row $X_i$ is the feature vector of node $i$. $N(i)$ refers to the set of nodes adjacent to node $i$. Let $A$ be the adjacency matrix of $\gG$ and $D$ be the diagonal matrix whose diagonal element $D_{ii}$ is the degree of node $i$. The normalized adjacency matrix is $\hat A = D^{-\frac{1}{2}}AD^{-\frac{1}{2}}$. Let $I$ denote the identity matrix. The \textit{normalized Laplacian matrix} $\hat L=I - \hat A$. Let $\hat L = U\Lambda U^T$ denote the eigendecomposition of $\hat L$, where $U$ is the matrix of eigenvectors and $\Lambda$ is the diagonal matrix of eigenvalues. 
\subsection{Graph Isomorphism}
A \textit{permutation} $\pi$ is a bijective mapping from $\{1, 2, ... , n\}$ to $\{1, 2, ..., n\}$, where $n\in \mathbb{N}^+$. For node set $\sV$, $\pi(\sV)=\{\pi(i)|i\in \sV\}$. For node feature matrix $X$, $\pi(X)_{\pi(i)}=X_i$. For edge set $\mathbb{E}$, $\pi(\mathbb{E})=\{(\pi(i),\pi(j))|(i,j)\in \mathbb{E}\}$. An \textit{automorphism} of a graph $\gG=(\sV, \mathbb{E})$ is a permutation $\pi$ such that $\pi(\sV)=\sV,\pi(\mathbb{E})=\mathbb{E}$. An automorphism of a graph with node features $\gG=(\sV, \mathbb{E}, X)$ is a permutation $\pi$ such that $\pi(\sV)=\sV,\pi(\mathbb{E})=\mathbb{E}, \pi(X)=X$. The \textit{order} of an automorphism is $\min_k \pi^k=e, k=1,2,...$, where $e$ is the identity mapping. Two nodes $i,j$ are \textit{isomorphic} if $\pi(i)=j$ under some automorphism $\pi$.

\subsection{Graph Signal Filter and Spectral GNNs}
The \textit{graph Fourier transform} of a signal $X\in\sR^{n\times d}$ is defined as $\tilde X = U^T X \in \sR^{n\times d}$. The \textit{inverse transform} is $X = U \tilde X$~\citep{GSP}. The $i^{\text{th}}$ column of $U$ is a frequency component corresponding to the eigenvalue $\lambda_i$. 

Let $\tilde X_\lambda=U_{:\lambda}^T X$, where $U_{:\lambda}$ is the eigenvector corresponding to $\lambda$, be the frequency component of $X$ at $\lambda$ frequency. If $\tilde X_\lambda\neq \vec 0$, we say $X$ contains the $\lambda$ frequency component. Otherwise, the $\lambda$ frequency component is missing from $X$.

We can use $g: [0,2]\to \sR$ to filter each frequency component by multiplying $g(\lambda)$. Applying a spectral filter $g$ on signal $X$ is defined as follows
\begin{equation}
\begin{aligned}
Ug(\Lambda)U^TX,
\end{aligned}
\end{equation}
where $g(\Lambda)$ applies $g$ element-wisely to the diagonal entries of $\Lambda$. 
To parameterize the filter, $g$ is often set to be a polynomial of degree $K$
\begin{equation}
\begin{aligned}
g(\lambda)&:=\sum_{k=0}^{K} \alpha_k \lambda^k.
\end{aligned}
\end{equation}
Then, the filtering process can be expressed by
\begin{equation}
\begin{aligned}
Ug(\Lambda)U^TX = \sum_{k=0}^{K}\alpha_k U\Lambda^k U^TX = \sum_{k=0}^{K}\alpha_k \hat L^k X.
\end{aligned}
\end{equation}
By defining $g(\hat L) = \sum_{k=0}^{K}\alpha_k \hat L^k$, we can rewrite the filtering process as follows:
\begin{equation}
\begin{aligned}
Ug(\Lambda)U^TX = g(\hat L) X.
\end{aligned}
\end{equation}
We show the forms of some popular spectral GNNs in Appendix~\ref{app::extMod}. In general, existing spectral-based GNNs can be unified into the following form:
\begin{equation}
\begin{aligned}
Z = \phi(g(\hat L)\varphi(X)),
\label{eq:general_form}
\end{aligned}
\end{equation}
where $Z$ is the prediction, $\phi$, $\varphi$ are functions like multi-layer perceptrons (MLPs) and $g$ is a polynomial. If a spectral-based GNN can express any polynomial filter function $g$, we call it \textit{Polynomial-Filter-Most-Expressive} (PFME) GNNs. We also define \textit{Filter-Most-Expressive} (FME) GNNs as the GNNs able to express arbitrary real-valued filter functions.

This study mainly focuses on the case when $\phi$ and $\varphi$ are linear functions, so we define \textit{linear GNN}.
\begin{definition}
A \textbf{linear GNN} can be formulated as $Z = g(\hat L)XW$, where $Z\in \sR^{n\times d'}$ is the prediction matrix, $g$ is a learnable real-valued polynomial, and $W\in \sR^{d\times d'}$ is a learnable matrix.
\end{definition}
Linear GNN keeps the spectral filter form of spectral GNNs despite its simplicity. And the expressive power of linear GNNs is a \textbf{lower bound} for that of general spectral GNNs in Equation~(\ref{eq:general_form}).
\begin{proposition}
Linear GNN is PFME. If $\phi$ and $\varphi$ can express all linear functions, spectral GNNs can differentiate any pair of nodes which linear GNNs can differentiate.
\end{proposition}

We assume all our models work on a fixed graph with fixed node features and only perform node property prediction tasks. Suppose there is an arbitrary real-valued filter function to approximate. Though PFME GNNs can only express polynomial filter functions, as the eigenvalue $\lambda$ is a discrete variable in a fixed graph, an interpolation polynomial always exists for the arbitrary filter and can produce the same output~\citep{Interpolation}. Therefore, PFME GNNs are FME in our problem setting. In the following analysis,
we always assume the linear GNNs have a high-enough degree $K$ in their polynomial filters. Linear GNNs with a limited degree $K$ are discussed in Appendix~\ref{app::limitedDegree}.

\section{Related Work}
\subsection{Spectral GNNs}
Spectral GNNs are GNNs based on spectral graph filters~\citep{SurveyGNN}. \citet{BernNet} categorize spectral GNNs by the filter operation adopted. One class is spectral GNNs with fixed filters: APPNP~\citep{APPNP} utilizes Personalized PageRank (PPR)~\citep{PPR} to build filter functions. GNN-LF/HF~\citep{GNN-LF} designs filter weights from the perspective of graph optimization functions. The other class is spectral GNNs with learnable filters: ChebyNet~\citep{ChebyConv} approximates the filters with Chebyshev polynomials. GPRGNN~\citep{GPRGNN} learns a polynomial filter by directly performing gradient descent on the polynomial coefficients. ARMA~\citep{ARMA} uses rational filters. BernNet~\citep{BernNet} expresses the filtering operation with Bernstein polynomials. Though ARMA~\citep{ARMA} and GNN-LF/HF~\citep{GNN-LF} use rational functions, they approximate the rational functions with polynomials. Therefore, all these methods use some form of polynomial filter despite the different bases they use. GPRGNN is one of the most expressive models. It can express all polynomial filters. So does ChebyNet, as the Chebyshev polynomials also form a complete set of bases in the polynomial space. They are both PFME. BernNet is less expressive as it forces the coefficients of the Bernstein polynomial bases to be positive and can only express positive filter functions. However, such constraints are introduced for regularization, so we ignore them when analyzing the expressive power. The filter forms of these models are summarized in Table~\ref{tab::filter_form}. 

\subsection{Removing Nonlinearity from GNNs}
Various GNNs removing nonlinearity have been proposed. \citet{SGC} precompute $\hat A^k X$ and perform logistic regression on the preprocessed features. Some works leverage personalized PageRank (PPR)~\citep{PPR} and random walk on the graph. \citet{GDC} use generalized graph diffusion, like the heat kernel and PPR, to reconstruct the graph. APPNP~\citep{APPNP} replaces normalized adjacency matrix with approximate PPR matrix to capture multi-hop neighborhood information. Some models with more complex acceleration techniques for computing PPR are introduced, like GBP~\citep{GBP}. Existing linear models are mainly motivated by improving the scalability and have restricted filters. In contrast, we analyze the expressive power and optimization property of linear GNNs with arbitrary polynomial filters.

\subsection{Expressive Power of GNNs}
The Weisfeiler-Lehman (WL) test of graph isomorphism~\citep{WLtest} is a series of algorithms that can distinguish almost all non-isomorphic graphs. Its $1$-dimensional form ($1$-WL) iteratively aggregates neighborhood labels and maps the aggregated labels into a new label for each node, which is similar to GNNs based on neighborhood node feature aggregation. The node labels assigned by $1$-WL test can also be used to check if two nodes are isomorphic~\citep{babai1979canonical,isomorphicnode}, as two isomorphic nodes always have the same label while two non-isomorphic ones mostly have different labels. \citet{HowPowerfulAreGNNs} show that the $1$-WL test bounds the expressive power of GNNs to distinguish non-isomorphic graphs. Since then, various works have attempted to analyze GNNs with the WL test and graph isomorphism testing~\citep{zhang2018end,WL1,WL2,WL3,WL5,WL6,rGIN,zhang2021labeling,zhang2021nested}. Other than the WL test and graph isomorphism, some works measure the expressive power in different ways, such as expressing universal invariant functions~\citep{functionapp1,functionapp2}, counting substructures~\citep{substrCount}, simulating Turing Machine~\citep{TuringUniv}, computing graph properties~\citep{GraphGeneralization}, and differentiating rooted graphs~\citep{GAMLP}. \citet{SpectralExpressive} analyze the expressive power from a spectral perspective, but the discussion is constrained to concluding former models to spectral forms. Our study provides conditions for spectral GNNs to approximate any functions and discuss the relation between these conditions and graph isomorphism. 

\section{Expressive Power of Linear GNNs}

In this section, we prove that linear GNNs are universal under three conditions and discuss these conditions to characterize how powerful spectral GNNs can reach. All proofs are in the appendix.

There are two components in a linear GNN $Z = g(\hat L)XW$:

\textbf{Linear Transformation $W$}. Since $XW=U(\tilde X W)$, a linear transformation in the spatial domain is also a linear transformation in the frequency domain, which produces a linear combination of signals in different channels.

\textbf{Filter $g(\hat L)$}. Since $g(\hat L)X= U(g(\Lambda)\tilde X)$, the filtering operation scales the frequency component corresponding to $\lambda$ of $\tilde X$ by $g(\lambda)$ fold in the frequency domain.

Now we give the \textbf{universality theorem} of linear GNNs.
\begin{theorem}\label{thr::linearexpressive}
{Linear GNNs can produce any one-dimensional prediction if $\hat L$ has no multiple eigenvalues and the node features $X$ contain all frequency components.}
\end{theorem}
There are three conditions for linear GNNs to be universal: 1) one-dimensional prediction, 2) no multiple eigenvalues, and 3) no missing frequency components. These are thus three bottlenecks for linear GNNs' expressive power. In the following, we discuss each of the three conditions in detail.

\subsection{Multidimensional Prediction}\label{sec:multiple_output_dim}
Though linear GNNs are powerful when the output has only one dimension, each dimension may need a different polynomial filter when the prediction has multiple channels. Take the toy graph in Figure~\ref{fig::multipred} as an example. One output dimension filters out high frequency signal and maintains low frequency signal, while the other one does the opposite. Therefore, a low-pass filter is needed for the first output dimension while a high-pass filter is needed for the second dimension. Using the same filter for all output dimensions cannot achieve this purpose.

We formally describe this property in Proposition~\ref{prop::multidim}.
\begin{proposition}\label{prop::multidim}
If the node feature matrix $X$ is not a full-row-rank matrix, for all $k>1$ and all graphs, there exists a $k$-dimensional prediction linear GNNs cannot produce.
\end{proposition}
We can use individual polynomial coefficients to compose a different filter for each output channel to solve this problem. 
\begin{figure}[ht]
\vskip -0.1in
\begin{center}
\centerline{\includegraphics[scale=0.36]{ 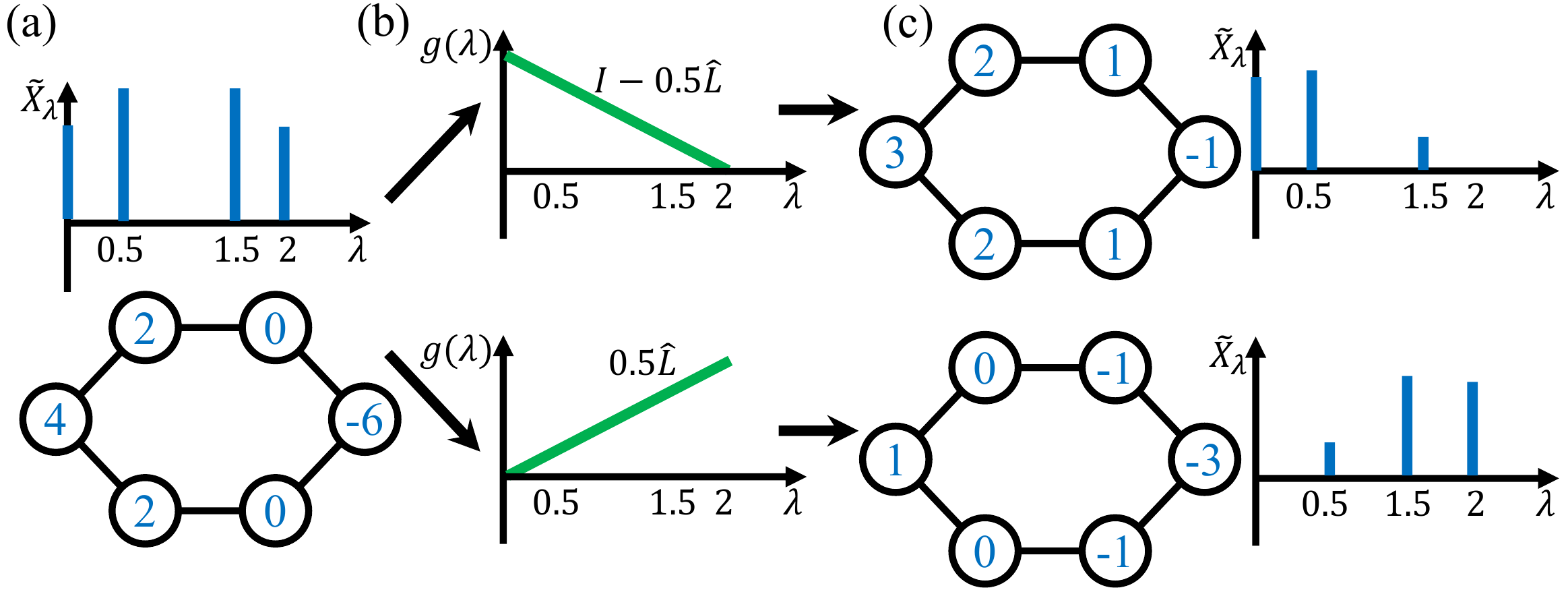}}
\end{center}
\vskip -0.4in
\caption{Individual filter function is needed for each prediction dimension. We illustrate each graph with both its spatial representation (where numbers on nodes represent one-dimensional node features) and its spectrum. (a) A graph with its node features. (b) Two different filters for two output dimensions. (c) Two output dimensions.}\label{fig::multipred}
\vskip -0.2in
\end{figure}
\subsection{Multiple Eigenvalue}
If two frequency components have the same eigenvalue $\lambda$, they will be scaled by the same number $g(\lambda)$. Therefore, the coefficients of these frequency components in prediction will keep the same ratio as in input $XW$. This issue is related to graph topology. More discussion is in Theorem~\ref{thr::singleeigen->permutation}.

\subsection{Missing Frequency Components}
The filter operation can only scale a frequency component. If this frequency component is missing from the node feature, the prediction cannot contain it either. Take the toy graph in Figure~\ref{fig::missingfreq} as an example. The node features only contain component corresponding to frequency $\lambda=0$, so a linear GNN cannot produce output with frequency $\lambda=2$ component. This problem is rooted in both the topology of graph $\gG$ and node features $X$ and is difficult to solve. 
\begin{figure}[ht]
\begin{center}
\centerline{\includegraphics[scale=0.42]{ 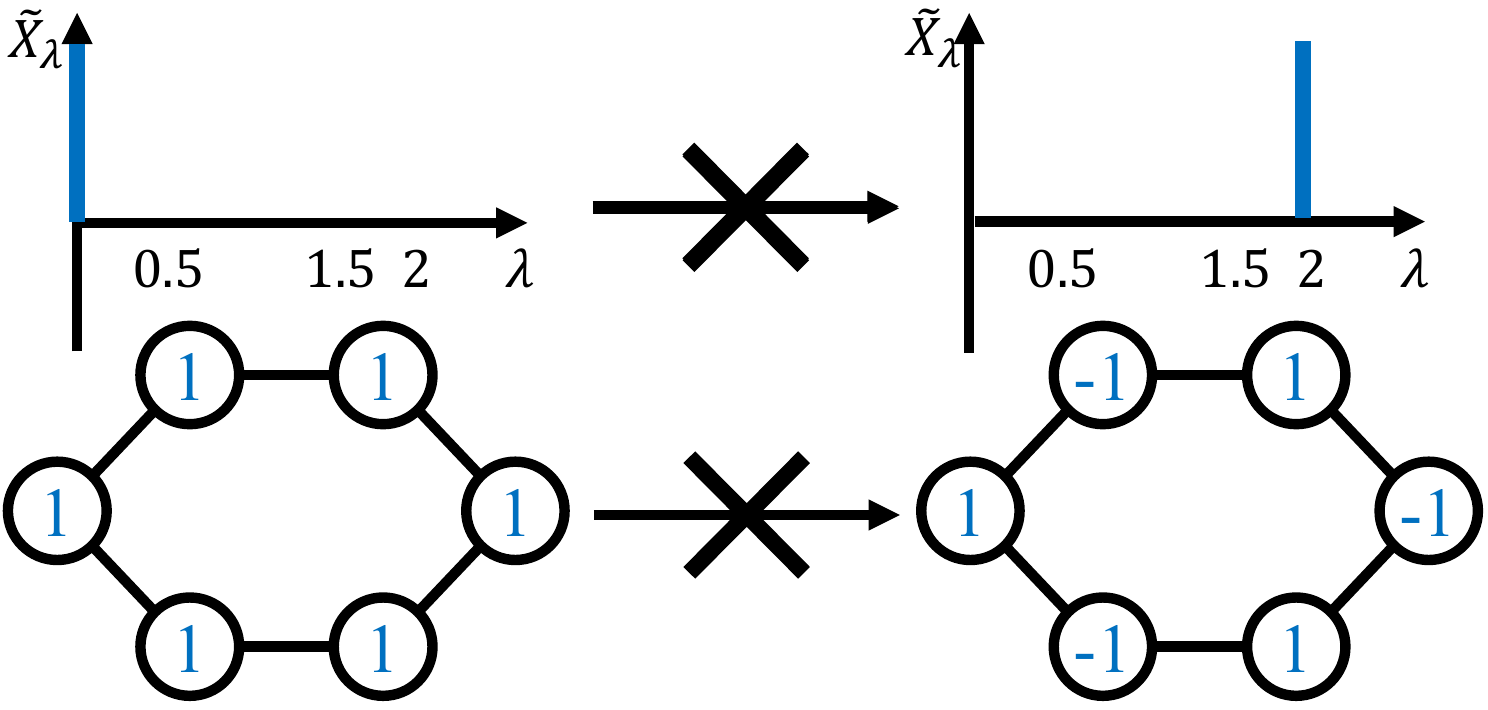}}
\end{center}
\vskip -0.4in
\caption{Node features with missing frequency components cannot produce some outputs.}\label{fig::missingfreq}
\vskip -0.1in
\end{figure}

Nevertheless, multiple eigenvalues and missing frequency components are both \textbf{rare in real-world graphs} with node features. See Appendix~\ref{app::datasets} for the ratio of multiple eigenvalues and number of missing frequency components in each of the 10 real-world benchmark datasets. In all the datasets, no frequency component is missing, and on average less than $1 \%$ of eigenvalues are multiple. Therefore, the universality conditions can be largely satisfied in practice.
\subsection{Connection to Graph Isomorphism}
Traditional expressivity analyses for spatial GNNs often leverage Graph Isomorphism testing. In this section, we explore the connections between our universality conditions and GI. We first build a connection between the expressive power of linear GNNs using a $K$-degree polynomial filter function and that of $(K+1)$-iteration WL test.

\begin{proposition}\label{prop::linear_vs_WL}
Given a linear GNN whose filter function is a $K$-degree polynomial, define the function $\text{LG}_K(i)$ as the prediction of node $i$ produced by the linear GNN. Let $\text{WL}_k(i)$ denote the label of node $i$ produced by $k$-iteration WL test whose initial label of node $i$ is the node feature vector $X_i$. Then $\forall i,j\in \sV$, $LG_K(i)=LG_K(j)$ if $\text{WL}_{K+1}(i)=\text{WL}_{K+1}(j)$. 
\end{proposition}
Proposition~\ref{prop::linear_vs_WL} means that linear GNNs' expressive power is also bounded by the $1$-WL test: if $1$-WL cannot differentiate two nodes, linear GNNs will also fail. However, this result seems to contradict with the universal approximation property of linear GNNs. We know that: 1) $1$-WL provably cannot discriminate some non-isomorphic nodes (such as nodes in a non-attributed regular graph), and 2) $1$-WL always gives isomorphic nodes the same label. However, a universal linear GNN should be able to give any two nodes different predictions, no matter whether they are isomorphic or not. To close this gap, we study the connections between the universality conditions of linear GNNs and the GI problem. Our results show that the no-multiple-eigenvalue and no-missing-frequency conditions enable $1$-WL to discriminate all non-isomorphic nodes, and also constrain the graph to contain no isomorphic nodes, therefore closing the gap.

We first show that $1$-WL can discriminate all non-isomorphic nodes under the two conditions.
\begin{corollary}\label{coro::1WLPower}
If a graph has no missing frequency component and its normalized Laplacian has no multiple eigenvalues, then $1$-WL can differentiate all non-isomorphic nodes.
\end{corollary}
The other part of the gap is that $1$-WL cannot produce different labels for isomorphic nodes, while linear GNNs with universal approximation property can. Therefore, we analyze how our no-multiple-eigenvalue and no-missing-frequency conditions constrain the graph in Theorems~\ref{thr::singleeigen->permutation} and \ref{thr::missingfreq}.
\begin{theorem}\label{thr::singleeigen->permutation}
For a graph whose normalized Laplacian has no multiple eigenvalues, the order of its automorphism is less than three.
\end{theorem}
The above theorem relates multiple eigenvalue to the degree of symmetry of the graph. A highly symmetric graph with three- or higher-order automorphism always has multiple eigenvalues. To intuitively understand this, we show two toy graphs in Figure~\ref{fig::permutation_singleeigen}.
The triangle (b) has a three-order automorphism and an eigenvector can be permuted to produce another linearly independent eigenvector, while for the two-node graph (a) with only two-order automorphism, permuting its eigenvector results in the same eigenvector. Since real-world graphs are often highly irregular, Theorem~\ref{thr::singleeigen->permutation} partly explains why multiple eigenvalues are rare in practice.
\begin{figure}[ht]
\vskip -0.1in
\begin{center}
\centerline{\includegraphics[scale=0.45]{ 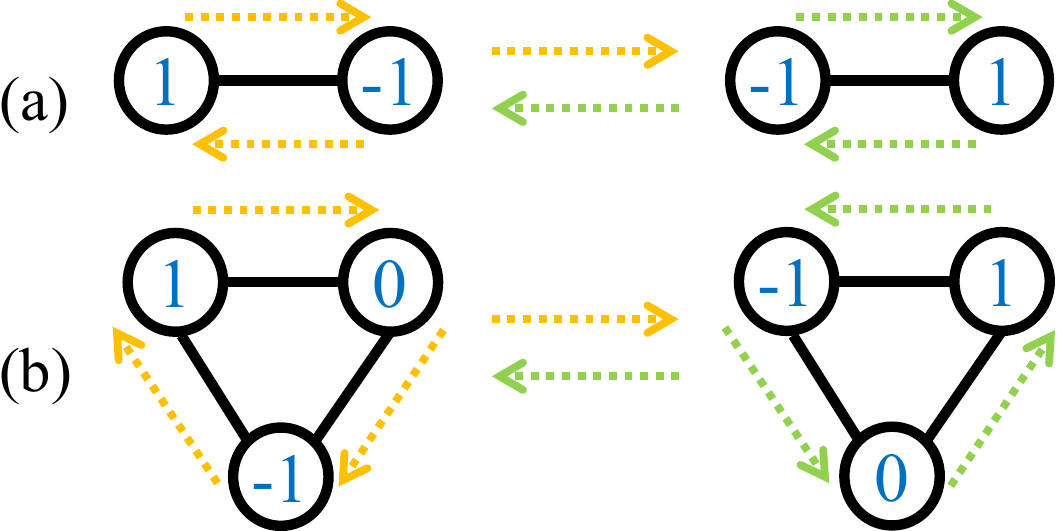}}
\end{center}
\vskip -0.4in
\caption{Graphs with high-order automorphisms (symmetries) have multiple Laplacian eigenvalues.}\label{fig::permutation_singleeigen}
\vskip -0.1in
\end{figure}

When considering node features containing all frequency components, all pairs of nodes are non-isomorphic, thus closing the gap between $1$-WL and linear GNN.
\begin{theorem}\label{thr::missingfreq}
Suppose a graph with node features does not have multiple eigenvalues in its normalized Laplacian, and no frequency component is missing from the node features. There will be no automorphism for the graph other than the identical mapping.
\end{theorem}
Therefore, the conditions of Theorem~\ref{thr::linearexpressive} constrain the graph topology and node features so that $1$-WL still bounds the expressive power of linear GNNs. On the other hand, our results indicate that $1$-WL can be quite powerful given \textbf{expressive node features and irregular graph structures}. Our results build a bridge between the expressive power of spectral GNNs (in terms of universality under some conditions) and spatial GNNs (in terms of $1$-WL test). As an analysis example, we also discuss how random features can boost the expressive power and why models with random features have poor empirical performance in Appendix~\ref{app::randDiscussion}. \citet{wang2022equivariant} also relates multiple eigenvalues to stability of positional encoding.

\subsection{Role of Nonlinearity}
Though linear GNNs have strong theoretical expressive power and remarkable empirical performance, various existing state-of-the-art GNNs utilize nonlinear activation functions. In this section, we analyze the role of nonlinearity. 

In linear GNNs, the $\lambda$ frequency component of the prediction $\tilde Z_\lambda$ is a function of only $g(\lambda)$, $\tilde X_\lambda$ and $W$. However, for nonlinear GNNs, different frequency components can transformed to each other. Figure~\ref{fig::nonlinear} is an example, where new frequency components emerge after ReLU activation. Consider an element-wise activation function $\sigma$ over the spatial signal $X$. We investigate its equivalent effect $\sigma'$ over the spectral signal $\tilde X$. Its function on a spectral signal is $\sigma'(\tilde X)=U^T\sigma(U\tilde X)$, meaning that different frequency components are first mixed via $U$, then nonlinearly transformed via $\sigma$ element-wisely, and finally distributed back to each frequency via $U^T$. Thus, $\sigma'$ is a column-wise nonlinear function over all frequency components. Mixing different frequency components may alleviate the issues from multiple eigenvalues and missing frequency components. However, such a mix is not expressive enough to solve all the problems, as $1$-WL still bounds the expressive power of GNNs. Furthermore, since the universality conditions are easily satisfied by real-world graphs to a large degree (Appendix~\ref{app::datasets}), we desert nonlinearity in our experiments.
\begin{figure}[ht]
\vskip -0.1in
\begin{center}
\centerline{\includegraphics[scale=0.42]{ 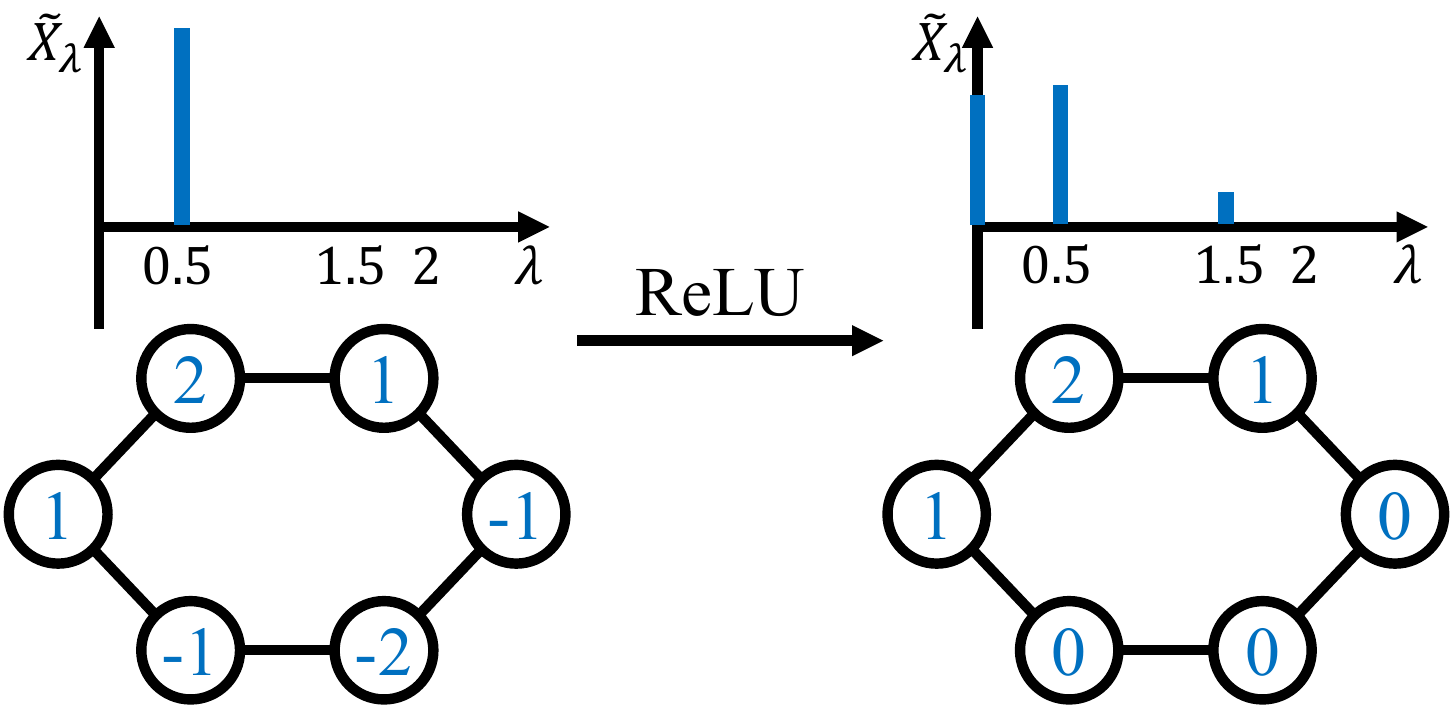}}
\end{center}
\vskip -0.4in
\caption{Nonlinear functions can mix different frequency components. }\label{fig::nonlinear}
\vskip -0.2in
\end{figure}

\subsection{Role of Bias}\label{sec::bias}
Bias is usually used together with a linear transformation. However, there is no need to discuss bias separately because,
\begin{equation}
    XW+b=\begin{bmatrix}X&1_n\end{bmatrix} \begin{bmatrix}W\\b\end{bmatrix},
\end{equation}
where row vector $b$ is the learnable bias, $1_n$ is a column vector whose all elements are $1$. Using bias is equivalent to adding $1$ to the features of each node and thus still keeps the linear GNN form. Therefore, we ignore bias in theoretical analysis, but by default turn on bias in experiments for better performance. As bias introduces extra graph signals, a natural question is whether bias can complete the frequency components missing from original node features. The answer is no. Please see Appendix~\ref{app::bias} for more details.

\section{Choice of Basis for Polynomial Filters}\label{sec::opt}
Assume the polynomial bases are $g_k(\lambda), k=0,1,2,...$ In this section, we discuss linear GNNs with individual filter parameters for each output dimension, which is formulated as
\begin{equation}
Z_{:l}=\sum_{k=0}^{K}\alpha_{kl}g_k(\hat L)XW_{:l}
\end{equation} 
where $\alpha_{kl}$ is the coefficient of polynomial filter basis $g_k(\hat L)$ and $(XW)_{:l}$ is the transformed node features for the $l^{\text{th}}$ output dimension $Z_{:l}$. 

All complete polynomial bases can build PFME models. However, models with different bases show different empirical performance. This section analyzes the effect of polynomial basis from an optimization perspective, which motivates the use of Jacobi Polynomials in our model.

\subsection{Hessian Matrix and Polynomial Basis}
Following the setting in~\citep{Optim}, we study linear GNNs trained with the squared loss $R=\frac{1}{2}||Z-Y||_F^2$, where $Y$ is the target. Assuming that linear GNNs can converge to the global minimum, we study the convergence speed near the global minimum. The rationality of this assumption is discussed in Appendix~\ref{app::dynamics}.

When considering the optimization of a linear GNN, both $\alpha$ and $W$ are learnable parameters. However, the gradient of loss over $W$ is a function of the learnable filter function $g_{:l}(\hat L):=\sum_{k}\alpha_{kl}g_k(\hat L)$ as a whole.
\begin{equation}
\begin{aligned}
\frac{\partial R}{\partial W_{jl}}&= \big[g_{:l}(\hat L)(XW)_{:l}-Y_{:l}\big]^T\big[g_{:l}(\hat L)X_{:j}\big],
\end{aligned}
\end{equation}
According to our assumption, the learned filter function is approximately the same for different bases as they have the same expressive power and can all converge to the global minimum. So the optimization of $W$ is irrelevant to the choice of basis near the global minimum. However, the optimization of $\alpha$ heavily depends on the basis choice. To focus on the effect of basis choice, we only analyze the optimization of $\alpha$ by merging $W$ into $X$.

Consider the optimization w.r.t. $\alpha$. The loss is a convex function, and the gradient descent's convergence rate depends on the Hessian matrix's condition number~\citep{ConvexOpt}. Therefore, we analyze the Hessian matrix of linear GNNs near the global minimum.

Since the total loss is summed over different output dimensions, and each output dimension adopts a different set of polynomial coefficients $\alpha_{kl}$, we can analyze the Hessian w.r.t. each dimension independently. Ignoring $l$, the $(k_1,k_2)$ element of the Hessian matrix $H$ can be written as
\begin{equation}
\begin{aligned}
\frac{\partial R}{\partial\alpha_{k_1}\partial\alpha_{k_2}}
&=X^Tg_{k_2}(\hat L)g_{k_1}(\hat L)X\\
&=\sum_{i=1}^n g_{k_2}(\lambda_i)g_{k_1}(\lambda_i)\tilde X_{\lambda_i}^2.
\end{aligned}
\end{equation}
It can be equivalently expressed as a Riemann sum:
\begin{equation}
\begin{aligned}
\sum_{i=1}^n g_{k_2}(\lambda_i)g_{k_1}(\lambda_i) \frac{F(\lambda_i)-F(\lambda_{i-1})}{\lambda_i-\lambda_{i-1}}(\lambda_i-\lambda_{i-1}),\\
\end{aligned}
\end{equation}
where $F(\lambda):=\sum_{\lambda_i\leq\lambda} \tilde X_{\lambda_i}^2$ is the accumulated amplitude of signal with frequency lower than $\lambda$. Define $f(\lambda)=\frac{\Delta F(\lambda)}{\Delta \lambda}$, which is the density of signal at frequency $\lambda$. In the limit when $n \rightarrow \infty$, we have:
\begin{equation}
\begin{aligned}
H_{k_1k_2}
&=\int_{\lambda=0}^{2} g_{k_1}(\lambda)g_{k_2}(\lambda)f(\lambda) \dd\lambda.\\
\end{aligned}
\end{equation}

The condition number $\kappa(H)$ reaches minimum if $H$ is an identity matrix, which is equivalent to that $g_k$'s form an orthonormal basis in the polynomial space whose inner product is defined by $\langle h,g\rangle=\int_{0}^{2} h(\lambda)g(
\lambda) f(\lambda)\dd\lambda$ with $f(\lambda)$ being the weight function. 

Our results show that although all complete polynomial bases have the same expressive power, using a set of orthonormal bases $g_k$ whose weight function corresponds to the graph signal density can enable linear GNNs to achieve the highest convergence rate. As the normalization of bases is straightforward, we only consider orthogonality in the analysis.

Given the weight function $f(\lambda)$, we can construct an orthonormal basis using the Gram-Schmidt process. However, the exact form of the weight function $f$ depends on the eigendecomposition of $\hat L$ and cannot be calculated efficiently and accurately for large graphs. Therefore, we choose a \textbf{general form} of orthogonal polynomials with \textbf{flexible enough} weight functions to adapt to different graph signal density functions $f(\lambda)$.

\subsection{Jacobi Polynomial Bases}
Among orthogonal polynomials, the Jacobi basis has a very general form, whereas the Chebyshev basis is a special case. The Jacobi basis $P_k^{a,b}$ has the following form.
\begin{equation}
\begin{aligned}
P_0^{a,b}(z)&=1,\\
P_1^{a,b}(z)&=\frac{a-b}{2}+\frac{a+b+2}{2}z.\\
\end{aligned}
\end{equation}
For $k\ge 2$.
\begin{align}
P_k^{a,b}(z)&=(\theta_{k} z+\theta'_{k}) P_{k-1}^{a,b}(z)
-\theta''_{k} P_{k-2}^{a,b}(z),
\end{align}
where 
\begin{equation}
\begin{aligned}
\theta_{k}&=\frac{(2k+a+b)(2k+a+b-1)}{2k(k+a+b)},\\
\theta'_{k}&=\frac{(2k+a+b-1)(a^2-b^2)}{2k(k+a+b)(2k+a+b-2)},\\
\theta''_{k}&=\frac{(k+a-1)(k+b-1)(2k+a+b)}{k(k+a+b)(2k+a+b-2)}.
\end{aligned}
\end{equation}
$P_k^{a,b}, k=0,1,2,...$ are orthogonal w.r.t. the weight function $(1-\lambda)^a(1+\lambda)^b$ on $[-1,1]$. We can define the Jacobi basis for graphs as $g_k(\hat L)=P_{k}^{a,b}(I-\hat L)=P_{k}^{a,b}(\hat A)$.

\subsection{A Discussion on Popular Filter Bases}\label{sec::anal_base}
In this section, we compare three popular polynomial bases with the Jacobi Polynomial: Monomial $(1-\lambda)^k$, Chebyshev $P_k^{-1/2,-1/2}(1-\lambda)$, 
and Bernstein $\tbinom{K}{k}(1-\frac{\lambda}{2})^{K-k}(\frac{\lambda}{2})^k$. These bases are visualized in Appendix~\ref{app::basesplot}.

For the Monomial basis, we can prove that it cannot be orthogonal on any weight function. 
\begin{proposition}\label{prop::power_non-ortho}
On any weight function $f(\lambda)$ which fulfils the requirements of the inner product, the Monomial basis is not orthogonal.
\end{proposition}
Chebyshev basis is a particular case of Jacobi basis and is only orthogonal w.r.t. a specific weight function. In contrast, the Jacobi basis can adapt to a wide range of weight functions. 

For non-orthogonal bases such as Bernstein, the Hessian matrix might not be diagonal, but a small condition number may still be achieved. In Appendix~\ref{app::Gram-Hessian}, we build a connection between the condition number of polynomial regression's Gram matrix using basis $g_k,k=0,1,2,.., K$ and that of linear GNNs' Hessian matrix. Therefore, some existing conclusions from polynomial regression basis choice can still be used. For example, existing studies show that the Bernstein basis can also achieve a lower condition number than the Monomial basis~\citep{BernLR}. Though both Bernstein and Jacobi basis can outperform Monomial, Jacobi basis can perform better if the weight function of Jacobi basis well approximates the data distribution. Our experiments find that the Jacobi basis outperforms the Bernstein basis on both synthetic and real-world datasets.

\section{JacobiConv Architecture}\label{sec::arch}
In this section, we describe our JacobiConv architecture. As the dimension of node features $X$ is often much larger than that of the transformed features $\hat X$, we first feed $X$ into a linear layer, $\hat X=XW+b$, with bias (see Section~\ref{sec::bias}), and then filter $\hat X$. There are three techniques used in the filter: multiple filter functions, Jacobi basis, and a novel polynomial coefficient decomposition (PCD) technique.
\subsection{Multiple Filters} 
Motivated by our analysis in Section~\ref{sec:multiple_output_dim}, we adopt an individual filter function for each output dimension. The JacobiConv can be formulated as
\begin{equation}
\begin{aligned}
Z_{:l}=\sum_{k=0}^{K}\alpha_{kl}P_k^{a,b}(\hat A)\hat X_{:l}.
\end{aligned}
\end{equation}
\subsection{Computation of Jacobi Basis}
With the recursion formula of Jacobi basis, we can compute all bases in $O(K)$ time and do $K$ message passing operations.
\begin{equation}
\begin{aligned}
P_0^{a,b}(\hat A)\hat X&=\hat X,\\
P_1^{a,b}(\hat A) \hat X&=\frac{a-b}{2}\hat X+\frac{a+b+2}{2}\hat A\hat X.\\
\end{aligned}
\end{equation}
For $k\ge 2$,
\begin{equation}
\begin{aligned}
P_k^{a,b}(\hat A)\hat X&=\theta_{k}\hat{A} P_{k-1}^{a,b}(\hat A)\hat X
+\theta'_{k}P_{k-1}^{a,b}(\hat A)\hat X\\
&-\theta''_{k}P_{k-2}^{a,b}(\hat A)\hat X.
\end{aligned}
\end{equation}
\subsection{Polynomial Coefficient Decomposition} 
The filter function we construct can be formulated as $\sum_{k=0}^{K} \alpha_{kl} P_{k}^{a,b}$. We find that in real-world datasets $\alpha_{kl}$ gets smaller as $k$ gets higher. As $\alpha_{kl}$'s have different magnitudes, the optimization can be hard. So we decompose $\alpha_{kl}$ to $\beta_{kl}\prod_{i=1}^k \gamma_i$, where $\gamma_i$'s are shared among different output channels. And we set $\gamma_i=\gamma'\tanh{\eta_i}$, which enforces $\gamma_i\in [-\gamma', \gamma']$. We call this technique Polynomial Coefficient Decomposition (PCD). We can modify the recursion formula to implement PCD.
\begin{align}
P_k^{a,b}(\hat A)\hat X&=\gamma_{k}\theta_{k}\hat{A} P_{k-1}^{a,b}(\hat A)\hat X
+\gamma_{k}\theta'_{k}P_{k-1}^{a,b}(\hat A)\hat X \nonumber \\
&-\gamma_{k}\gamma_{k-1}\theta''_{k}P_{k-2}^{a,b}(\hat A)\hat X.
\end{align}
\section{Experiment}
In this section, we first conduct experiments on synthetic datasets to examine JacobiConv's ability to express filter functions, and then test JacobiConv on real-world datasets. Our code is
available at \url{https://github.com/GraphPKU/JacobiConv}.
\subsection{Evaluating Models on Learning Filters}
Following \citet{BernNet}, we transform real images to 2D regular 4-neighbor grid graphs, whose nodes are pixels. We apply 5 spectral filters (low $e^{-10\lambda^2}$, high $1-e^{-10\lambda^2}$, band $e^{-10(\lambda-1)^2}$, reject $1-e^{-10(\lambda-1)^2}$, and comb $|\sin \pi\lambda|$) to the signal in each image. All models use original graph signal as node features to fit the filtered signal. 

\begin{table}[t]
\centering
\caption{Average of sum of squared loss over 50 images.}\label{tab::filter}
\vskip 0.1in
\begin{center}
\begin{small}
\begin{sc}
\setlength{\tabcolsep}{1mm}
{\begin{tabular}{lccccc}
\hline
           & Low      & High     & Band     & Reject & Comb           \\ \hline
GPRGNN    & $0.4169$ & $0.0943$ & $3.5121$ & $3.7917$  & $4.6549$ \\
ARMA       & $1.8478$ & $1.8632$ & $7.6922$ & $8.2732$ & $15.1214$ \\
ChebyNet   & $0.8220$ & $0.7867$ & $2.2722$ & $2.5296$  & $4.0735$  \\
BernNet    & $0.0314$ & $0.0113$ & $0.0411$ & $0.9313$ & $0.9982$  \\
\hline
JacobiConv	&$\bf{0.0003}$&$\bf{0.0011}$&$\bf{0.0213}$&$\bf{0.0156}$&$\bf{0.2933}$\\\hline
Monomial	&$2.4076$	&$4.2411$	&$10.8856$	&$8.7031$	&$10.5596$\\
Chebyshev	&$0.9227$	&$2.3198$	&$7.7751$	&$6.0065$	&$9.1191$	\\
Bernstein	&$0.0110$	&$0.0058$	&$0.1517$	&$0.1607$	&$0.5705$	\\
\hline
\end{tabular}}
\end{sc}
\end{small}
\end{center}
\vskip -0.05in
\end{table}

We report the average squared error (lower the better) over the 50 pictures. Results are shown in Table~\ref{tab::filter}. For a fair comparison, we remove PCD from linear GNNs. 

We compare JacobiConv with popular PFME GNNs: GPRGNN~\citep{GPRGNN}, ARMA~\citep{ARMA}, BernNet~\citep{BernNet}, and ChebyNet~\citep{ChebyConv}. Settings of these models are detailed in Appendix~\ref{app::experimentsetting}. JacobiConv outperforms other models on all datasets and even achieves up to $50$ times lower loss on two datasets: Low and Reject. Though all these models can learn arbitrary polynomial filters, JacobiConv has better optimization properties as it uses orthogonal filter bases that can adapt to a wide range of signal distributions.

We also compare linear GNNs with different bases. The results are shown in the lower part of Table~\ref{tab::filter}. Jacobi basis still outperforms other bases on all datasets and achieves $10$ times lower loss than any other basis. Bernstein basis also achieves lower loss than Monomial on all datasets, which verifies our analysis in Section~\ref{sec::anal_base}.

\begin{figure}[t]
\begin{center}
\centerline{\includegraphics[scale=0.25]{ 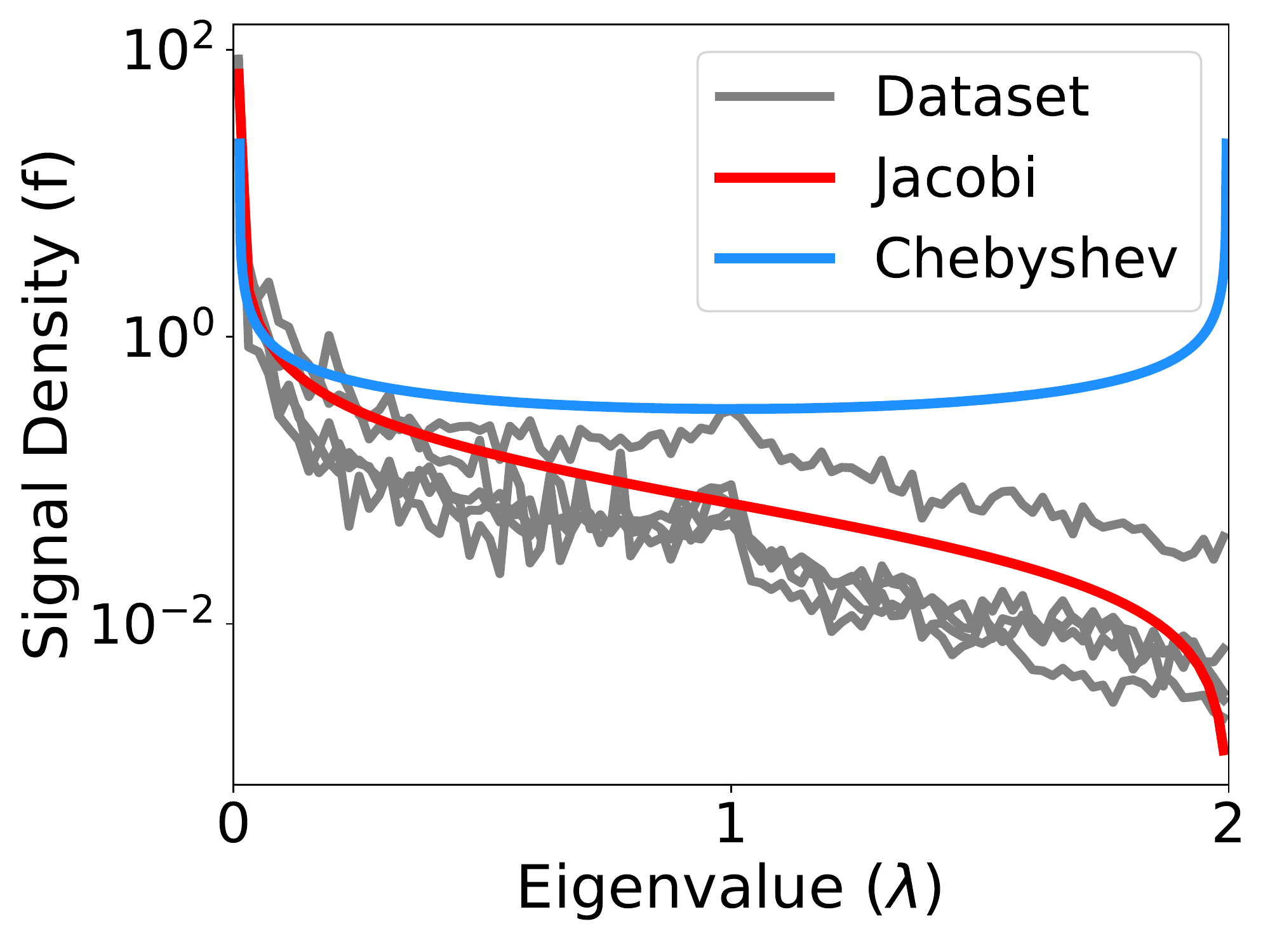}}
\vskip -0.1in
\caption{Signal density functions of some graphs in the image dataset and the weight functions of some bases.}\label{fig::sigplot}
\end{center}
\vskip -0.4in
\end{figure}

To verify that Jacobi basis adapts to the dataset, we plot the signal distributions of some randomly selected graphs in our image dataset and the weight functions of Jacobi and Chebyshev bases in Figure~\ref{fig::sigplot}. We can see that only Jacobi basis (with hyperparameters selected for minimizing loss) can capture the main shape of the signal distribution, compared to Chebyshev basis.

Experimental results of models with PCD on synthetic datasets are shown in Appendix~\ref{app::filter+PCD}. JacobiConv still outperforms any other model on all datasets. Jacobi basis also achieve a higher convergence rate than other bases for linear GNN. See Appendix~\ref{app::ImgOptim} for the convergence rate. 

\subsection{Evaluation on Real-World Datasets}
For homogeneous graphs, we include three citation graph datasets, Cora, CiteSeer and PubMed~\citep{Cora}, and two Amazon co-purchase graphs, Computers and Photo~\citep{Photo}. We also use heterogeneous graphs, including Wikipedia graphs Chameleon and Squirrel~\citep{Chameleon}, the Actor co-occurrence graph, and the webpage graph Texas and Cornell from WebKB3~\citep{Texas}. Their statistics are listed in Appendix~\ref{app::datasets}. We perform the node classification task, where we randomly split the node set into train/validation/test sets with a ratio of $60\%/20\%/20\%$. JacobiConv is compared with spectral GNNs: GCN, APPNP, ChebyNet, GPRGNN, and BernNet. Note that all these baselines use \textbf{nonlinear} transformations, while JacobiConv is a purely \textbf{linear} model. Results are shown in Table~\ref{tab::real}. Settings of these models are detailed in Appendix~\ref{app::experimentsetting}.

\begin{table*}[t]
\centering
\caption{Results on real-world datasets: Mean accuracy (\%) $\pm$ $95$\% confidence interval.}\label{tab::real}
\vskip 0.1in
\begin{center}
\begin{small}
\begin{sc}
\resizebox{0.7\textwidth}{!}{
\setlength{\tabcolsep}{1mm}
\begin{tabular}{lcccccccc}
\hline
Datasets  & GCN             & APPNP           & ChebyNet        & GPRGNN         & BernNet              & JacobiConv \\
\hline
Cora      & $87.14_{\pm 1.01}$ & $88.14_{\pm 0.73}$ & $86.67_{\pm 0.82}$ & $88.57_{\pm 0.69}$ & $88.52_{\pm 0.95}$      & $\bf{88.98_{\pm 0.46}}$  \\
Citeseer  & $79.86_{\pm 0.67}$ & $80.47_{\pm 0.74}$ & $79.11_{\pm 0.75}$ & $80.12_{\pm 0.83}$ & $80.09_{\pm 0.79}$      & $\bf{80.78_{\pm 0.79}}$\\
pubmed    & $86.74_{\pm 0.27}$ & $88.12_{\pm 0.31}$ & $87.95_{\pm 0.28}$ & $88.46_{\pm 0.33}$ & $88.48_{\pm 0.41}$      & $\bf{89.62_{\pm 0.41}}$ \\
Computers & $83.32_{\pm 0.33}$ & $85.32_{\pm 0.37}$ & $87.54_{\pm 0.43}$ & $86.85_{\pm 0.25}$ & $87.64_{\pm 0.44}$      & $\bf{90.39_{\pm 0.29}}$ \\
Photo     & $88.26_{\pm 0.73}$ & $88.51_{\pm 0.31}$ & $93.77_{\pm 0.32}$ & $93.85_{\pm 0.28}$ & $93.63_{\pm 0.35}$      & $\bf{95.43_{\pm 0.23}}$ \\
Chameleon & $59.61_{\pm 2.21}$ & $51.84_{\pm 1.82}$ & $59.28_{\pm 1.25}$ & $67.28_{\pm 1.09}$ & $68.29_{\pm 1.58}$      & $\bf{74.20_{\pm 1.03}}$ \\
Actor     & $33.23_{\pm 1.16}$ & $39.66_{\pm 0.55}$ & $37.61_{\pm 0.89}$ & $39.92_{\pm 0.67}$ & $\bf{41.79_{\pm 1.01}}$ & $41.17_{\pm 0.64}$\\
Squirrel  & $46.78_{\pm 0.87}$ & $34.71_{\pm 0.57}$ & $40.55_{\pm 0.42}$ & $50.15_{\pm 1.92}$ & $51.35_{\pm 0.73}$      & $\bf{57.38_{\pm 1.25}}$\\
Texas     & $77.38_{\pm 3.28}$ & $90.98_{\pm 1.64}$ & $86.22_{\pm 2.45}$ & $92.95_{\pm 1.31}$ & $93.12_{\pm 0.65}$      & $\bf{93.44_{\pm 2.13}}$\\
Cornell   & $65.90_{\pm 4.43}$ & $91.81_{\pm 1.96}$ & $83.93_{\pm 2.13}$ & $91.37_{\pm 1.81}$ & $92.13_{\pm 1.64}$      & $\bf{92.95_{\pm 2.46}}$ \\
\hline
\end{tabular}
}
\end{sc}
\end{small}
\end{center}

\end{table*}
\begin{table*}[t]
\vskip -0.1in
\centering
\caption{Results of ablation study on real-world datasets: Mean accuracy (\%) $\pm$ $95$\% confidence interval.}\label{tab::abl}
\begin{center}
\begin{small}
\begin{sc}
\resizebox{1\textwidth}{!}{
\setlength{\tabcolsep}{1mm}{
\begin{tabular}{lcccc|cccccc}
\hline
Datasets & Monomial & Chebyshev & Bernstein & Jacobi & JacobiConv & UniFilter & No-PCD &NL-Res &NL\\
\hline
Cora 
& $88.80_{\pm0.67}$ & $88.49_{\pm0.82}$ & $86.50_{\pm1.26}$ & $\bf{88.98_{\pm0.72}}$ 
& $88.98_{\pm0.46}$ & $\bf{89.05_{\pm0.48}}$ & $88.98_{\pm0.72}$ & $89.00_{\pm0.61}$ & $88.67_{\pm0.69}$ \\
Citeseer
& $\bf{80.68_{\pm0.86}}$ & $80.53_{\pm0.81}$ & $80.61_{\pm0.85}$ & $80.61_{\pm0.72}$ 
 & $\bf{80.78_{\pm0.79}}$ & $80.42_{\pm0.98}$ & $80.61_{\pm0.71}$ & $80.16_{\pm0.86}$ & $80.25_{\pm0.60}$ \\
Pubmed 
& $89.54_{\pm0.36}$ & $89.52_{\pm0.46}$ & $88.42_{\pm0.32}$ & $\bf{89.70_{\pm0.34}}$ 
& $89.62_{\pm0.41}$ & $89.58_{\pm0.25}$ & $\bf{89.70_{\pm0.34}}$ & $86.44_{\pm2.05}$ & $87.73_{\pm2.13}$ \\
Computers
& $89.06_{\pm0.24}$ & $89.16_{\pm0.47}$ & $87.09_{\pm0.38}$ & $\bf{89.22_{\pm0.39}}$ 
& $90.39_{\pm0.29}$ & $\bf{90.45_{\pm0.34}}$ & $89.22_{\pm0.42}$ & $87.45_{\pm2.15}$ & $86.85_{\pm2.67}$ \\
Photo 
& $95.33_{\pm0.25}$ & $95.45_{\pm0.27}$ & $94.59_{\pm0.26}$ & $\bf{95.53_{\pm 0.27}}$ 
& $95.43_{\pm0.23}$ & $95.26_{\pm0.31}$ & $\bf{95.53_{\pm0.19}}$ & $94.16_{\pm0.78}$ & $85.65_{\pm8.25}$ \\
Chameleon 
& $65.95_{\pm1.20}$ & $\bf{74.09_{\pm0.85}}$ & $70.24_{\pm1.05}$ & $72.95_{\pm0.83}$ 
& $\bf{74.20_{\pm1.03}}$ & $73.76_{\pm1.03}$ & $72.95_{\pm0.83}$ & $72.63_{\pm0.99}$ & $72.56_{\pm1.01}$ \\
Actor 
& $40.31_{\pm0.82}$ & $40.61_{\pm0.64}$ & $40.42_{\pm0.50}$ & $\bf{40.70_{\pm0.98}}$ 
& $\bf{41.17_{\pm0.64}}$ & $40.01_{\pm0.96}$ & $40.70_{\pm0.98}$ & $37.80_{\pm1.32}$ & $37.56_{\pm0.88}$ \\
Squirrel 
& $37.93_{\pm0.62}$ & $\bf{56.71_{\pm0.89}}$ & $44.48_{\pm0.89}$ & $55.77_{\pm0.55}$ 
& $\bf{57.38_{\pm1.25}}$ & $54.11_{\pm0.82}$ & $55.77_{\pm0.55}$ & $48.66_{\pm6.65}$ & $43.73_{\pm6.94}$ \\
Texas 
& $91.64_{\pm2.46}$ & $88.36_{\pm3.93}$ & $89.34_{\pm2.46}$ & $\bf{92.79_{\pm1.97}}$ 
& $\bf{93.44_{\pm2.13}}$ & $90.82_{\pm2.30}$ & $92.79_{\pm1.97}$ & $89.84_{\pm3.28}$ & $89.34_{\pm3.12}$ \\
Cornell 
& $91.31_{\pm2.13}$ & $88.03_{\pm3.28}$ & $\bf{92.46_{\pm2.63}}$ & $92.30_{\pm2.79}$ 
& $\bf{92.95_{\pm2.46}}$ & $92.62_{\pm2.46}$ & $92.30_{\pm2.62}$ & $89.67_{\pm2.30}$ & $87.54_{\pm3.11}$ \\ \hline
\end{tabular}}
}
\end{sc}
\end{small}
\end{center}
\vskip -0.18in
\end{table*}

JacobiConv outperforms all existing models on $9$ out of $10$ datasets and achieves performance gains up to $12\%$ on a heterogeneous dataset Squirrel. On the Actor dataset, JacobiConv beats all baselines except BernNet. The generally top and runner-up performance of JacobiConv and BernNet verify our analysis in Section~\ref{sec::anal_base}. The results indicate that JacobiConv is a general spectral GNN with consistently good performance across datasets. They also show that nonlinearity is not necessary for learning powerful spectral filters given a good choice of polynomial basis.
\subsection{Ablation Analysis}\label{sec::abl}

To illustrate the effectiveness of Jacobi basis, we compare JacobiConv with linear GNNs with other filter bases in the left part of Table~\ref{tab::abl}. We also remove PCD from the models to ensure fairness as the coefficient distribution of different bases varies. Jacobi basis outperforms any other basis by more than $0.8\%$ on average. Bernstein basis also outperforms Monomial on average, which is consistent with the results in Table~\ref{tab::real}. 

In the right part of Table~\ref{tab::abl}, UniFilter is JacobiConv using the same filter for all prediction dimensions. No-PCD is JacobiConv without PCD. The results illustrate that the multiple filter functions, PCD, and the Jacobi basis are all essential for JacobiConv. On average, the multiple filter technique provides $1.3\%$ performance gain, and the PCD technique provides $0.8\%$ performance gain.

We design two variants to analyze how removing nonlinearity affects performance: NL and NL-Res. NL replaces the linear transformation in JacobiConv with a $2$-layer ReLU MLP, whose first-layer output has the same dimension as the model output dimension. Compared with NL, NL-Res uses residual connection, which adds the output of the first linear layer to the output of the MLP. NL-Res outperforms NL by $2\%$ on average, while NL leads to $6\%$ performance loss compared with JacobiConv. These results illustrate that linear GNN is expressive enough, and nonlinear transformations can hardly promote the expressive power. The better performance of NL-Res over NL might also be due to its closer relationship to linear GNNs. On the other hand, the lower performance after adding nonlinearity may be attributed to overfitting caused by extra parameters. 

\begin{table}[t]
\centering
\caption{Parameters/per-epoch time (ms)/total training time (s).}
\vskip -0.1in
\setlength\tabcolsep{1pt}
\label{tab::runtime}
\begin{center}
\begin{small}
\begin{sc}
\resizebox{0.49\textwidth}{!}{
\begin{tabular}{lccccc}
\hline
Datasets        & JacobiConv & APPNP  & BernNet & GPRGNN \\ \hline
cora &10K/6.4/3.1 &~92K/3.6/1.2 &~92K
/11.6/3.1 &~92K/4.3/0.9 \\
citeseer &22K/6.3/3.0 &237K/3.7/1.3 &237K/11.8/3.4 &237K/4.5/1.0 \\
pubmed &~2K/6.6/4.9 &~32K/3.9/2.0 &~32K/11.1/4.9 &~32K/4.5/1.8 \\
computers &~8K/7.3/4.8 &~50K/6.0/2.5 &~50K/29.3/8.6 &~50K/6.5/1.6 \\
photos &~6K/6.4/4.8 &~48K/5.8/2.8 &~48K/15.3/6.2 &~48K/4.5/1.3\\
chameleon &12K/6.5/4.4 &149K/3.9/0.8 &149K/11.0/2.8 &149K/4.4/1.0 \\
actor &~5K/6.5/3.4 &~60K/3.8/0.8 &~60K/10.9/3.5 &~60K/4.3/0.9 \\
squirrel &11K/6.3/6.1 &134K/4.3/0.9 &134K/15.7/4.9 &134K/4.3/2.1 \\
texas &~9K/6.6/3.4 &109K/3.8/0.8 &109K/11.3/2.4 &109K/4.3/1.0 \\
cornell &~9K/6.5/3.4 &109K/3.8/0.8 &109K/11.0/2.4 &109K/4.4/0.9 \\ \hline
\end{tabular}
}
\end{sc}
\end{small}
\end{center}
\vskip -0.6in
\end{table}
\subsection{Scalability}
As shown in Table~\ref{tab::runtime}, compared with other baselines with comparable depth, our model, on average, only uses $10\%$ parameters, as it only uses a linear layer to convert node features to the output shape, while other models use MLPs. JacobiConv also has a similar computational overhead to other baselines, though taking more time than APPNP and GPRGNN due to more complex bases. Theoretically, it still has the same time complexity $O(Kmd)$ as APPNP and GPRGNN, where $K$ is the degree of the polynomial, $m$ is the number of edges in the graph, and $d$ is the number of node feature dimensions, while BernNet's time complexity is $O(K^2md)$.
\section{Conclusion}
In this paper, we analyze the expressive power of spectral GNNs. We prove that even without nonlinearity, spectral GNNs can be universal under mild conditions. We further analyze the optimization of spectral GNNs, which motivates the proposed JacobiConv, a novel spectral GNN using Jacobi basis. JacobiConv outperforms the previous state-of-the-art method BernNet by up to $12\%$ on real-world datasets without using nonlinearity, which verifies our theory. 

\subsection*{Acknowledgements} 
The authors greatly thank the actionable suggestions from the reviewers. Zhang is partly supported by the CCF-Baidu Open Fund (NO.2021PP15002000).


\newpage
\bibliography{example_paper}
\bibliographystyle{icml2022}

\newpage
\appendix
\onecolumn

\section{Existing Models}\label{app::extMod}
\begin{table*}[h]
\centering
\caption{The filter form of spectral GNNs.}\label{tab::filter_form}
\vskip 0.15in
\begin{center}
\begin{small}
\begin{sc}
\begin{tabular}{lcccc}
\hline
Model      & $g$ & Hyperparams & Learnable & PFME \\ \hline
SGC~\textnormal{\citep{SGC}}      &    $(1-\lambda)^K$         &    $\alpha,K$        &                      &            $\times$            \\
APPNP~\textnormal{\citep{APPNP}}      &    $\sum_{k=0}^K \frac{\alpha^k}{1-\alpha} (1-\lambda)^k$         &    $\alpha,K$        &                      &            $\times$            \\
GNN-LF~\textnormal{\citep{GNN-LF}}     &    $\frac{1-(1-\mu)(1-\lambda)}{1-(2-\mu+\frac{1}{\alpha})(1-\lambda)}$         &  $\alpha,\mu$          &                      &              $\times$          \\
GNN-HF~\textnormal{\citep{GNN-LF}}     &    $\frac{1+\beta(1-\lambda)}{1-(1-\beta-\frac{1}{\alpha})(1-\lambda)}$         &  $a,b$          &                      &              $\times$          \\
ChebyNet~\textnormal{\citep{ChebyConv}}   &    $\sum_{k=0}^K \alpha_k \cos(k\arccos(1-\lambda))$         &   $K$     &                  $\alpha_k, K$    &         $\surd$               \\
GPRGNN~\textnormal{\citep{GPRGNN}}     &    $\sum_{k=0}^K \alpha_k (1-\lambda)^k$         &   $K$            &         $\alpha_k$          &                   $\surd$     \\
ARMA~\textnormal{\citep{ARMA}}       &       $\sum_{k=0}^K \frac{b_k}{1-a_k(1-\lambda)}$      &       $K$     &                      $a_k,b_k$&  $\surd$                      \\
BernNet~\textnormal{\citep{BernNet}}    &        $\sum_{k=0}^K \alpha_k \tbinom{K}{k} (1-\frac{\lambda}{2})^{K-k}(\frac{\lambda}{2})^k$     &        $K$    &            $\alpha_k$          &     $\surd$                   \\
JacobiConv (our model) &      $\sum_{k=0}^K \alpha_k \sum_{s=0}^k \frac{(k+a)!(k+b)!(-\lambda)^{k-s}(2-\lambda)^s}{2^ks!(k+a-s)!(b+s)!(k-s)!} $     &     $K, a, b$       &       $\alpha_k$               &      $\surd$                \\ \hline
\end{tabular}
\end{sc}
\end{small}
\end{center}
\vskip -0.1in
\end{table*}
\section{Proofs}\label{app::proofs}

\subsection{Proof of Theorem~\ref{thr::linearexpressive}}\label{app::prooflinearexpressive}

We restate Theorem~\ref{thr::linearexpressive} as follows.
\begin{theorem}
Assuming all rows of $\tilde X$ are not zero vector, and no eigenvalue of $\hat L$ has multiplicity larger than $1$, for all $Z\in \sR^{n\times 1}$, there exists a linear GNN to produce it.
\end{theorem}
{\it Proof.} First, we prove that $W^*\in \sR^{d}$ exists so that all elements of $\tilde XW^*$ are not zero. 

Consider the $i^{\text{th}}$ row of $\tilde XW$ equals $0$. In other words, $\tilde X_{i}W=0$. Let the solution space of $W$ be $V_i$. As $\tilde X_i\neq 0$, $V_i$ is a proper subspace of $\sR^{d}$. Therefore, $\sR^{d}-\bigcup_{i=1}^n V_{i}\neq \emptyset$. All vectors $W$ in $\sR^{d}-\bigcup_{i=1}^n V_{i}\neq \emptyset$ can meet the requirements, 

Then we filter $\tilde XW^*$ to produce the output. For all one-dimension prediction $Z\in \sR^{n}$, $\tilde Z=U^TZ\in \sR^{n}$. If there exists a polynomial that $g^*(\lambda_i)=R_i$, where $R$ is a vector whose $i^{\text{th}}$ row $R_i=\frac{\tilde Z_i}{(\tilde XW)_i}$, for $i\in \{1, 2, ..., n\}$ , linear GNNs can produce $Z$.

As $\lambda_i$ are different from each other, consider an $n-1$ degree polynomial, $g(\lambda_i)=\sum_{k=0}^{n-1}\theta_k\lambda_i^k$. The coefficient $\theta_k$ of $g^*$ is the solution of the linear system $B\Theta=R$, where $B\in\sR^{n\times n}$ and $B_{ij}=\lambda_{i}^{j-1}$, $\Theta \in \sR^n$ and $\Theta_k=\theta_{k-1}$, $R\in \sR^{n}$, gives the coeffcient of $g$. As $B^T$ is a Vandermonde matrix and becomes nonsingular if eigenvalues are different from each other, a solution always exists. Therefore, linear GNNs can give arbitrary one-dimensional prediction.

\rightline{\qedsymbol}

\subsection{Proof of Proposition~\ref{prop::multidim}}\label{app::multidim}
Assuming an output $Z\in\sR^{n\times k}, k>1$, that linear GNNs can express it is equivalent to that the equation $Z=g(\hat L)XW$ has solution polynomial $g$ and matrix $W$. The equation is equivalent to $\tilde Z=g(\Lambda)\tilde X W$. Let $\tilde X_{s_i},i=1,2,...,\text{rank}(\tilde X)$ be a maximal linearly independent subset of the set of row vectors in $\tilde X$. We prove that linear GNNs cannot produce the prediction described in the following lemma.

\begin{lemma}
Assuming that all the elements of the $s_i^{\text{th}}$ row of $\tilde Z$ are the same scalar $\tilde Z_{s_i}\in \sR-\{0\}$, $i=1,2,...,n$ and there exists $\tilde Z_{ij_1}\neq \tilde Z_{ij_2}$, where $i\in\{1,2,...,n\}-\{s_i|i=1,2,...,\text{rank}(X)\}, j_1, j_2\in \{1,2,...,k\}, j_1\neq j_2$, no linear GNN can produce $U\tilde Z$.
\end{lemma}
{\it Proof.} $\tilde X=U^TX$, where $U$ is an orthogonal matrix. Therefore, $\text{rank}(\tilde X)=rank(X)<n$.

Let $\sI$ denote the set $\{s_i|i=1,2,...,\text{rank}(X)\}$. As $\tilde X_{sI}$ forms a maximal linearly independent row vectors of $\tilde X$. Therefore, there exists $M\in \sR^{n\times\text{rank(X)}}, \tilde X= M\tilde X_{\sI}$. Only consider the rows in $\sI$ of the equation.
\begin{equation}
\begin{aligned}
\tilde Z_{\sI}=g(\Lambda)_{\sI\sI}\tilde{X}_{\sI}W.\\
\end{aligned}
\end{equation}
As all elements in $\tilde{Z_{\sI}}\neq 0$, all diagonal elements of $g(\Lambda)_{\sI\sI}\neq 0$. Therefore,
\begin{equation}
\begin{aligned}
g(\Lambda)_{\sI\sI}^{-1}\tilde{Z}_{\sI}=\tilde{X}_{\sI}W.
\end{aligned}
\end{equation}
Therefore, all column vectors of $\tilde Z$ should be equal, because
\begin{equation}
\begin{aligned}
\tilde{Z}=g(\Lambda)\tilde{X}W=g(\Lambda)M\tilde{X}_{\sI}W
=(g(\Lambda)M g(\Lambda)_{\sI}^{-1})\tilde{Z}_{\sI}.
\end{aligned}
\end{equation}
As all column vectors of $\tilde Z_{\sI}$ are equal, column vectors of $\tilde Z$ are all the same, while we assume that there exists $i\in\{1,2,...,n\}-\{s_i|i=1,2,...,\text{rank}(X)\}, j_1, j_2\in \{1,2,...,n\}, j_1\neq j_2$ that $\tilde Z_{ij_1}\neq \tilde Z_{ij_2}$. Therefore, such linear GNNs do not exist.

\rightline{\qedsymbol}

\subsection{Proof of Proposition~\ref{prop::linear_vs_WL} and Corollary~\ref{coro::1WLPower}}\label{app::prof::prop::linear_vs_WL}
{\it Proof.}
When the filter function is a $K$-degree polynomial, the prediction of the linear GNN can be formulated as follows.
\begin{equation}
\begin{aligned}
Z=\sum_{k=0}^K\theta_k \hat A^k(XW).
\end{aligned}
\end{equation}
Using the framework in~\citep{HowPowerfulAreGNNs}, it can be considered as a $K+1$-layer GNN. Let $h^{(k)}_i$ denote the embeddings of node $i$ at the $k^{\text{th}}$ layer. $\text{COMBINE}^{(k)}$, $\text{AGGREGATE}^{(k)}$ are functions defined as follows. 
\begin{equation}
\begin{aligned}
&a^{(1)}_i=\text{AGGREGATE}^{(1)}(\{h^{(k-1)}_j|j\in N(i)\})=|\{h^{(k-1)}_j|j\in N(i)\}|=D_{ii}\\
&\text{COMBINE}^{(1)}(a^{(1)}_i, X_{i})=(D_{ii}, \theta_K X_{i}, X_{i}),\\
\end{aligned}
\end{equation}
where $\text{COMBINE}^{(1)}$ produce a tuple containing three items. 
For $k=2,..., K$,
\begin{equation}
\begin{aligned}
&a^{(k)}_i=\text{AGGREGATE}^{(k)}(\{(D_{jj},h^{(k-1)}_j, X_i)|j\in N(i)\})=\sum_{j\in N(i)} \frac{1}{\sqrt{D_{jj}}}h^{(k-1)}_j\\
&\text{COMBINE}^{(k)}(a^{(k)}, (D_{ii},h^{(k-1)}_j,X_i))=(D_{ii}, \frac{1}{\sqrt{D_{ii}}}a^{(k)}_i+\theta_{K+1-k}X_{i}, X_{i}).\\
\end{aligned}
\end{equation}
For $k=K+1$,
\begin{equation}
\begin{aligned}
&a^{(k)}_i=\text{AGGREGATE}^{(k)}(\{(D_{jj},h^{(k-1)}_j, X_i)|j\in N(i)\})=\sum_{j\in N(i)} \frac{1}{\sqrt{D_{jj}}}h^{(k-1)}_j\\
&\text{COMBINE}^{(k)}(a^{(k)}, (D_{ii},h^{(k-1)}_j,X_i))= \frac{1}{\sqrt{D_{ii}}}a^{(k)}_i+\theta_{0}X_{i}.\\
\end{aligned}
\end{equation}
Therefore, the output of the last layer in GNN produce the output of linear GNNs. According to the proof of Lemma 2 in \citet{HowPowerfulAreGNNs}, if WL node labels $WL_k(v)=WL_k(u)$, we always have GNN node features $h^{(k)}_i=h^{(k)}_j$ for any iteration $i$. Therefore, for all nodes $i, j\in \sV$, $LG_K(i)=LG_K(j)$ if $WL_{K+1}(i)=WL_{K+1}(j)$.

The proof of Corollary~\ref{coro::1WLPower} is obvious. For any pair of non-isomorphic nodes in the graph, linear GNNs can produce different outputs for the two nodes, so $1$-WL can also differentiate them.

\subsection{Proof of Theorem~\ref{thr::singleeigen->permutation}}\label{app::singleeigen->graph}
Assuming $\pi$ is a permutation function and $P$ is a permutation matrix, $\delta_{\pi(a),a}$, the graph is isomorphic under the permutation $\pi$.
\begin{equation}
\begin{aligned}
\hat{L}&=P^T\hat LP\\
U\Lambda U^T&=PU\Lambda U^TP^T\\
\Lambda&=U^TPU\Lambda U^TP^TU\\
\Lambda&=V\Lambda V^T,\\
\end{aligned}
\end{equation}
where $V$ is an orthogonal matrix. As all diagonal elements of $\Lambda$ are different, the eigenspace corresponding to each eigenvalue has only one dimension. Therefore, 
\begin{equation}
\begin{aligned}
U^TPU=V=D',
\end{aligned}
\end{equation}
where $D'$ is a diagonal matrix whose diagonal elements are $\pm 1$. Therefore, \begin{equation}
\begin{aligned}
P=UD'U^T.
\end{aligned}
\end{equation}
Therefore, $P$ is symmetric, in other words, for $i\in\{1,2,...,n\}$, $\pi(\pi(i))=i$. Therefore, for any graph without multiple normalized Laplacian eigenvalue, the order of permuatations is $1$ or $2$.

\subsection{Proof of Theorem~\ref{thr::missingfreq}}

For all permutation $\pi$ and its matrix $P$ for graph.
\begin{equation}
\begin{aligned}
\hat A&=P\hat AP^T\\
X&=PX
\end{aligned}
\end{equation}

Let $V$ denote $U^TPU$.

\begin{equation}
\begin{aligned}
\Lambda&=V\Lambda V^T\\
\tilde X&=V\tilde X
\end{aligned}
\end{equation}

If $\hat A$ does not have multiple eigenvalues, $V=D$, $D$ is a diagonal matrix whose diagonal elements are $\pm 1$. So $(I-D)\tilde X=0$.

Assuming all rows of$\tilde X$ are not zero vector (no missing frequency component), $I-D=0$, $D=I$. 

Therefore, $P=UDU^T=I$. Therefore, all pairs of nodes in this graph are not isomorphic when considering node features.

\subsection{Proof of Proposition~\ref{prop::power_non-ortho}}

Orthogonality require $\langle x,x\rangle\neq 0$ while $\langle 1, x^2\rangle =0$. However,
\begin{equation}
\begin{aligned}
\langle 1, x^2\rangle = \int_{0}^2 x^2f(x)\dd x=\langle x, x\rangle.
\end{aligned}
\end{equation}

\subsection{Proof of Proposition~\ref{prop::randomSignal}}\label{app::randomSignal}
As $\tilde X= U^TX$, the distribution density $f_1$ of $\tilde X$ has a simple relation with the distribution density function $f_2$ of $X$,
\begin{equation}
\begin{aligned}
f_1(\tilde X)&= f_2(U\tilde X)|\det(U^T)|\\
&=\frac{1}{\det({2\pi\sigma^2 I})^{1/2}}e^{-\frac{1}{2}\tilde X^TU^T(\sigma^2 I)^{-1}U\tilde X}\\
&=\frac{1}{\det({2\pi\sigma^2 I})^{1/2}}e^{-\frac{1}{2}\tilde X^T(\sigma^2 I)^{-1}\tilde X}
\end{aligned}
\end{equation}
Therefore, $\tilde X\sim N_n(0,\sigma^2 I)$.

We can extend this proposition to the multi-dimensional cases. Consider $X\in\sR^{n\times d}$, $\text{vec}(X)\in N_{nd}(0,\sigma^2 I)$, $\tilde X = U^TX$, $\text{vec} (\tilde X)=I\bigotimes U^T \text{vec}(X)$. $I\bigotimes U^T$ is still a orthogonal matrix. Therefore, $\text{vec}(\tilde X)\in N_{nd}(0,\sigma^2 I)$

\subsection{Proof of Theorem~\ref{thr::random_expressive}}\label{app::random_expressive}
We use a lemma from ~\citep{randommat_rank}.
\begin{lemma}
Let $F(x_1,...,x_m)$ be a non-zero polynomial of variables $x_1,...,x_m$ with real coefficients, then, $\mu_mD = 0$, where $D= \{x|F(x)=0,x=(x_1,...,x_m)^T\in \sR^m\}$ and $\mu_mD$ is the Lebesgue measure of $D$ as the set of points in $\sR^m$.
\end{lemma}
As we use individual filter parameters for each output dimension, if we can produce arbitrary one-dimensional prediction, muli-dimensional prediction can also be produced. So we can assume $Z\in\sR^{n\times 1}$. Consider the linear GNNs in the frequency domain.
\begin{equation}
\begin{aligned}
\tilde Z = g(\Lambda)\tilde X W.
\end{aligned}
\end{equation}

Assuming that multiple eigenvalues are in the $i_1, i_2,...$ rows of $\Lambda$. Let $\sI$ be $i_1, i_2,...$, $|\sI|=\sum s_i$. As no frequency components are missing from $Z$, all diagonal elements in $g(\Lambda)$ are not zero. 

We first build $g(\Lambda)_{\sI\sI}$ and $W_{\sI}$ to produce $Z_{\sI}$.
\begin{equation}
\begin{aligned}
\tilde Z_{\sI} = g(\Lambda)_{\sI\sI}\tilde X_{\sI} W.
\end{aligned}
\end{equation}
As all elements in $\tilde X$ independently follows $N(0,\sigma^2)$, the probability that $\tilde X_{\sI}$ becomes singular is,
\begin{equation}
\begin{aligned}
\int_{|\tilde X_{\sI}|=0} \frac{1}{(2\pi\sigma^2)^{d^2/2}} e^{-\frac{1}{2\sigma^2}||X_{\sI}||_F^2} \dd \tilde X_{\sI}
\le \int_{|\tilde X_{\sI}|=0} \frac{1}{(2\pi\sigma^2)^{d^2/2}} \dd \tilde X_{\sI}=0.
\end{aligned}
\end{equation}
Therefore, $W=(\tilde X_{\sI})^{-1}g(\Lambda)_{\sI\sI}^{-1}\tilde Z_{\sI}\neq 0$. With probablity $1$, $Z_{\sI}$ can be produced.

Then we consider how to produce other rows. Let $\sJ=\{1,2,..., n\}-\sI$.
\begin{equation}
\begin{aligned}
\tilde Z_{\sJ}=g(\Lambda_{\sJ\sJ}) \tilde X_{\sJ}W
\end{aligned}
\end{equation}
The probability of some rows of $\tilde X_{\sJ}W$ are $0$ is,
\begin{equation}
\begin{aligned}
\int_{\min_{i\in \sJ} |\tilde X_{i'}W|=0} \frac{1}{(2\pi\sigma^2)^{d(n-d)/2}} e^{-\frac{1}{2\sigma^2}||X_{\sJ}||_F^2} \dd \tilde X_{\sJ}
&\le \sum_{i'=1}^{n-1}\int_{|\tilde X_{i'}W|=0} \frac{1}{(2\pi\sigma^2)^{d(n-d)/2}} e^{-\frac{1}{2\sigma^2}||X_{\sJ}||_F^2} \dd \tilde X_{\sJ}\\
&\le \sum_{i'=1}^{n-1}\int_{|\tilde X_{i'}W|=0} \frac{1}{(2\pi\sigma^2)^{(n-d)d/2}} \dd \tilde X_{\sJ}=0.
\end{aligned}
\end{equation}

Assume that all rows of $\tilde X_{\sJ}W$ are not zero. As all elements of $\Lambda_{\sJ\sJ}$ are different, we can let $g(\Lambda)_{i_j}=\tilde Z_{i_j}/(\tilde X_{i_j}W)$. Therefore, other rows of $\tilde Z$ can also be built with probablity $1$. The probability that $Z$ can be produced is $1$.

\subsection{Proof of Proposition~\ref{prop::randomPolyDegree}}\label{app::randomPolyDegree}

The number of different eigenvalues is $O(n)$. Let $\sI=\{i_1,i_2,...\}$ be the set of the index of different eigenvalues, and $\lambda_{i_1}<\lambda_{i_2}<...$.

Consider the signal $\tilde x$ in the frequency domain. $\tilde x\sim N(0,\sigma^2 I)$. For any pair of adjacent elements in $\tilde x_{\sI}$, $\tilde x_{i_j}$ and $\tilde x_{i_{j+1}}$, the probability that two nodes have different signs is $\frac{1}{2}$. There, $O(n)$ pairs of $\tilde x_{i_j}$ and $\tilde x_{i_{j+1}}$ have different signs.

For these pairs, after filtering, $\tilde z_{i_j}=g(\lambda_{i_j})\tilde x_{i_j}$, $\tilde z_{i_{j+1}}=g(\lambda_{i_{j+1}})\tilde x_{i_{j+1}}$. There are three cases.

\begin{itemize}
    \item $\tilde z_{i_j}=0$ or $\tilde z_{i_{j+1}}=0$. A zero-point exist for $g$.
    \item $\tilde z_{i_j}$ and $\tilde z_{i_{j+1}}$ have the same signs. A zero-point exist for $g$ in $(\lambda_{i_j},\lambda_{i_{j+1}})$.
    \item $\tilde z_{i_j}$ and $\tilde z_{i_{j+1}}$ have different signs. 
\end{itemize}

Therefore, the number of zero points of $g$ is the number of case $1$ add that of case $2$ minus the count of case $3$, $O(n)-O(1)=O(n)$.

Therefore, the degree of polynomial is $O(n)$ in expectation.

\section{Polynomial Filter with Limited Degree}\label{app::limitedDegree}

Approximating functions with polynomials is well studied in numerical analysis. Weierstrass Approximation Theorem ensures the asymptotic approximation. For fixed-order polynomials, Theorem 3.3 of \citet{Interpolation} shows that, when approximating a filter function $h\in C^{n+1}[0,2]$ with an $n$-order polynomial $g(x)$, an upper bound for the error exists.
\begin{equation}
sup_{x\in[0, 2]}|h(x)-g(x)| \le \frac{1}{(n+1)!}(sup_{x\in [0, 2]}|h^{(n+1)}(x)|)(sup_{x\in [0, 2]}|\prod_{i=0}^n(x-x_i)|),
\end{equation}
where $x_0, x_1,..., x_n$ are distinct numbers selected in $[0, 2]$. Let $x_i$ be Chebyshev points $1+\cos(\frac{2i+1}{2n+2}\pi)$. 

\begin{align}
sup_{x\in[0, 2]}|h(x)-g(x)| &\le \frac{1}{(n+1)!}(sup_{x\in [0, 2]}|h^{(n+1)}(x)|)(sup_{x\in [0, 2]}|\frac{1}{2^{n}}\cos((n+1)\arccos(x-1))|)\\
&\le \frac{1}{(n+1)!2^n}sup_{x\in [0, 2]}|h^{(n+1)}(x)|.
\end{align}

Therefore, the approximation error of polynomial depends on both the polynomial degree and the property of filter function. In linear GNNs, as each output dimension learns a different filter, we consider only one output dimension. The squared loss is bounded as follows.
\begin{align}
    \frac{1}{2}||Y-Z||_F^2&=\frac{1}{2}(Z-Y)^T(Z-Y)\\
    &= \frac{1}{2}(U\tilde Y-U\tilde Z)^T(U\tilde Y-U\tilde Z)\\
    &=\frac{1}{2}||\tilde Y-\tilde Z||_F^2\\
    &=\frac{1}{2}||(h(\Lambda)-g(\Lambda))\tilde XW||_F^2\\
    &\le \frac{1}{2}(\sup_{\lambda\in [0,2]}|h(\lambda)-g(\lambda)|)^2 ||\tilde XW||_F^2\\
    &\le\frac{1}{2}  (\frac{1}{(n+1)!2^{n}})^2  ||XW||_F^2 sup_{x\in [0, 2]}|h^{(n+1)}(x)|^2.
\end{align}

\section{Random Feature. Why? Why not?}\label{app::randDiscussion}
Next, we study ways to break the no-missing-frequency condition in Theorem~\ref{thr::linearexpressive} to increase linear GNNs' expressive power.

Existing literature has tried to utilize random features for GNNs. GNN-RNI~~\citep{GNN-RNI} randomly initializes node embeddings and can approximate any functions mapping graphs to real numbers. \citet{GNN-RNI} prove that GNN with random features can universally approximate any permutation invariant function $f: \gG_n\to \sR$, which mainly describes the representation of the whole graph. \citet{rGIN} prove that GNN with random features can distinguish any local structure. Both works analyze from a graph isomorphism perspective. However, from a spectral perspective, we prove the expressive power of random features for node property tasks and analyze why it fails on node classification tasks.

First, we prove that no frequency component is missing from the random feature.
\begin{proposition}\label{prop::randomSignal}
Assume vector $x\sim N_n(0, \sigma^2 I)$, where $N_n$ is the Gaussian distribution of $n$ variables. The graph Fourier transformation of $x$ is $\tilde x\sim N_n(0,\sigma^2 I)$.
\end{proposition}
The proposition is proved in Appendix~\ref{app::randomSignal}. 

We call $x$ in Proposition~\ref{prop::randomSignal} random features. Therefore, the probability of some frequency components missing from the random features is $0$. If we concatenate random features to the node features, no frequency components will be missing from the node features. 
Moreover, random features can also help with the multiple eigenvalue problem. 
\begin{theorem}\label{thr::random_expressive}
Assuming that the number of multiple eigenvalues is $m$, and among them, the $i-\text{th}$ multiple eigenvalue has multiplicity $s_i$. With $(\sum_{i=1}^m s_i)$-dimensional $\sim N_n(0, \sigma^2 I)$ random node features, for all prediction with no missing frequency components, linear GNNs can produce it with probability $1$.
\end{theorem}
The proof can be found in Appendix~\ref{app::random_expressive}.

We have seen the power of random features for improving the expressive power of linear GNNs. However, on large graphs, this technique can worsen the performance of models. As the coefficient of components of node features vibrates frequently, the filter function may be very complex even if we fit simple graph signals. Therefore, as formalized in Proposition~\ref{prop::randomPolyDegree}, $O(n)$-degree polynomial is needed, which is impossible to implement for large graphs.

\begin{proposition}\label{prop::randomPolyDegree}
{ If $\hat L$ has no multiple eigenvalue, $O(n)$ degree polynomial is needed for linear GNN using Gaussian random features to predict a one-dimensional non-zero target whose coeffcients of frequency components can be expressed as a $O(1)$-degree polynomial.}
\end{proposition}
The proof of Proposition~\ref{prop::randomPolyDegree} can be found in Appendix~\ref{app::randomPolyDegree}.

$O(n)$-degree polynomial is too time- and memory-consuming for real-world datasets. In practice, we can only afford constant-degree polynomials (such as degree $10$ in our experiments), which explains why random features usually worsen the performance.

To verify our analysis, we compare JacobiConv (our proposed model) with random features (Random), JacobiConv with learnable random features (Learnable), and the original JacobiConv in Table~\ref{tab::randExp}. Random features significantly worsen the performance, while learnable random features performs much better. Much to our surprise, Learnable even beats JacobiConv on two datasets, which indicates that node features may have little useful information in some datasets.

\begin{table}[t]
\centering
\caption{Results on real-world datasets: Mean accuracy (\%) $\pm$ $95$\% confidence interval.}\label{tab::randExp}
\begin{tabular}{cccccc}
\hline
Datasets        & JacobiConv & Learnable &Random     \\ \hline
Cora     & $\bf{88.98_{\pm0.46}}$ & $82.82_{\pm 0.59}$& $23.46_{\pm 1.92}$\\
CiteSeer & $\bf{80.78_{\pm0.79}}$ & $ 71.17_{\pm 1.20}$&$18.99_{\pm 1.02}$ \\
Pubmed   & $\bf{89.62_{\pm0.41}}$  & $83.26_{\pm 0.46}$ & $35.90_{\pm 0.83}$\\
Computers & $\bf{90.45_{\pm0.34}}$ & $83.69_{\pm 0.30}$ & $18.40_{\pm 2.25}$\\
Photo     & $\bf{95.43_{\pm0.23}}$ & $91.87_{\pm 0.19}$ & $19.90_{\pm 2.71}$\\
Chameleon & ${74.20_{\pm1.03}}$ & $\bf{75.10_{\pm 0.83}}$&$25.03_{\pm 2.30}$ \\
Actor     & $\bf{41.17_{\pm0.64}}$ & $24.50_{\pm 1.10}$ & $25.88_{\pm 0.43}$\\
Squirrel  & ${57.38_{\pm1.25}}$ & $\bf{61.00_{\pm 1.09}}$ & $20.26_{\pm 0.73}$\\
Texas     & $\bf{93.44_{\pm2.13}}$ & $72.13_{\pm 5.41}$  & $28.36_{\pm 5.57}$\\
Cornell   & $\bf{92.95_{\pm2.46}}$ & $31.31_{\pm 11.48}$ & $25.74_{\pm 7.21}$ \\ \hline
\end{tabular}
\vskip -0.1in
\end{table}

\section{Can Bias Complete Missing Components?}\label{app::bias}
Missing components hamper the expressive power of linear GNNs. Adding a bias to the linear transformation may alleviate this problem, as it can introduce new components. However, bias cannot solve this problem completely.

\begin{proposition}
There exists a graph of size $n$ and $X\in \sR^{n\times d}$ with missing components such that $\forall b\in \sR^{1\times d'}, \forall W\in \sR^{d\times d'}$, some frequency components are still missing from $XW+b$.
\end{proposition}
\begin{proof}
Consider a graph $\gG$ of size $n$ whose $0$, $1$ nodes are isolated. Let $\gS_1$ denote the subgraph composed of the two isolated nodes. $\gS_2$ means the subgraph composed of nodes $\{2,...,n-1\}$.  Let $\hat L$, $\hat L_{\gS_1}$, $\hat L_{\gS_2}$ denote the normalized Laplacian of $\gG$, $\gS_1$, $\gS_2$, respectively. We have
\begin{align}
    \hat L_{\gS_1} &= \begin{bmatrix}\frac{\sqrt{2}}{2}&\frac{\sqrt{2}}{2}\\-\frac{\sqrt{2}}{2}&\frac{\sqrt{2}}{2}\end{bmatrix}\begin{bmatrix}1&0\\0&1\end{bmatrix}\begin{bmatrix}\frac{\sqrt{2}}{2}&-\frac{\sqrt{2}}{2}\\\frac{\sqrt{2}}{2}&\frac{\sqrt{2}}{2}\end{bmatrix}\\
    \hat L_{\gS_2} &= U_{\gS_2}\Lambda_{\gS_2}U_{\gS_2}^T\\
    \hat L &= \text{diag}(L_{\gS_1}, L_{\gS_2})\\
    \hat L &= U_{\gG}\Lambda U_{\gG}^T,
\end{align}
where $U_{\gG}^T =\text{diag}( \begin{bmatrix}\frac{\sqrt{2}}{2}&-\frac{\sqrt{2}}{2}\\\frac{\sqrt{2}}{2}&\frac{\sqrt{2}}{2}\end{bmatrix}, U_{\gS_2}^T)$. Therefore, $\forall b$, the $0^{\text{th}}$ row of $U_{\gG}^T1_nb$ is $\vec 0$. Let $\tilde X_{0}=0$. Therefore, the $0^{\text{th}}$ row of $U^T(XW+b)$ is $\vec 0$. In other words, some components are missing from $XW+b$.
\end{proof}

\section{Connection between Linear GNN and Polynomial Regression.}\label{app::Gram-Hessian}

Let $\tilde F(x)=\frac{1}{F(2)}F(x)$. Take $n'$ independent random variables $x_1,x_2,...,x_n'$ from the $\tilde F$ distribution and set these variables as the points of linear regression. The element of the Gram matrix $G'$ of the polynomial regression using $g_k, k=0,2,...,K$ basis is,
\begin{equation}
\begin{aligned}
G'_{k_1k_2} = \sum_{i=1}^{n'} g_{k_1}(x_i)g_{k_2}(x_i).\\
\end{aligned}
\end{equation}
By the weak law of large numbers of probability theory,
\begin{equation}
\begin{aligned}
\lim_{n'\to\infty} \frac{G'_{k_1k_2}}{n'} = {\mathcal {E}} \left[\frac{G'_{k_1k_2}}{n'}\right]=\int_{0}^2 g_{k_1}(x)g_{k_2}(x) \tilde f(x) \dd x.\\
\end{aligned}
\end{equation}
Therefore, $n'\to \infty$, $\frac{G'_{ij}}{n}F(2)=G_{ij}$. Scalar $\frac{F(2)}{n}$ will not affect condition number. When $n\to \infty$, $\kappa(G')$, the condition number of Gram matrix of polynomial regression using bases $g_k,k=0,1,2,.., K$ with points sampled from $\tilde F$ distribution, equals to, $\kappa(H)$, the condition number of linear GNNs' Hessian matrix. Therefore, we can use some conclusions on polynomial regression's bases.

\section{Datasets}\label{app::datasets}

We summarize the statistics of these datasets in Table~\ref{tab::datasets}.

\begin{table}[h]
\centering
\caption{Dataset statistics. $N_{\text{miss}}$ is the number of missing frequency components. $R_{\text{multi}}$ is the ratio of multiple eigenvalues in all different Laplacian eigenvalues (\%).}\label{tab::datasets}
\vskip 0.15in
\begin{center}
\begin{small}
\begin{sc}
\setlength{\tabcolsep}{1mm}
{\begin{tabular}{ccccccccccc}
\hline
Datasets & Cora & CiteSeer & PubMed & Computers & Photo  & Chameleon & Squirrel & Actor & Texas & Cornell \\
\hline
$|V|$ & 2708 & 3327 & 19717 & 13752 & 7650 & 2277 & 5201 & 7600 & 183 & 183     \\
$|E|$ & 5278 & 4552 & 44324 & 245861 & 119081 & 31371 & 198353& 26659 & 279 & 277     \\
$N_{\text{miss}}$ & 0 &0 &0& 0 &0 &0& 0 &0 &0& 0\\
$R_{\text{multi}}$ & 1.49&3.25&0.03&0.16&0.22&0.22&0.00&0.13&0.00&0.00\\
\hline
\end{tabular}}
\end{sc}
\end{small}
\end{center}
\vskip -0.1in
\end{table}

\section{Experimental Settings}\label{app::experimentsetting}

{\bf Computing infrastructure.}~~We leverage Pytorch Geometric and Pytorch for model development. All experiments are conducted on an Nvidia A40 GPU on a Linux server. 

{\bf Baselines.}~~We directly use the results reported in \citep{BernNet}. JacobiConv and linear GNN with other bases have fewer parameters than baselines, as linear GNN have a fixed number of parameters given the node feature dimension and output dimension.

{\bf Model hyperparameter for Synthetic Datasets.}~~We use optuna to perform random searches. Hyperparameters were selected to minimize average loss on the fifty images. The best hyperparameters selected for each model can be found in our code. For linear GNNs, we use different learning rate and weight decay for the linear layer $W$, parameters of PCD $\theta$, and the linear combination parameters $\alpha$. We select learning rate from $\{0.0005, 0.001, 0.005, 0.01, 0.05\}$, weight decay from $\{0.0, 5e-5, 1e-4, 5e-4, 1e-3\}$. We select PCD's $\gamma$ from $\{0.5, 1.0, 1.5, 2.0\}$.  Jacobi Basis' $a$ and $b$ are selected from $[-1.0, 2.0]$. 

{\bf Model hyperparameter for real-world datasets.}~~ Hyperparameters were selected to optimize accuracy scores on the validation sets. We use different dropout for $X$ and $XW$. Both dropout probabilities are selected from $[0.0, 0.9]$. Other parameters are searched in the same way as synthetic datasets. 

{\bf Training process.} We utilize Adam optimizer to optimize models and set an upper bound ($1000$) for the number of forward and backward processes. An early stop strategy is used, which finishes training if the validation score does not increase after $200$ epochs for real-world datasets. 

\section{Synthetic Dataset Results of Models with PCD}\label{app::filter+PCD}

Results are shown in Table~\ref{tab::filter+PCD}. JacobiConv still outperforms all other bases. In general, there is little performance difference between models with PCD and those without. The reason for the invalidation of PCD can be that linear GNN trained with the squared loss on synthetic datasets can converge to a global minimum, and the effect of PCD to help convergence may be unimportant.

\begin{table}[t]
\centering
\caption{Average of sum of square loss over 50 images.}\label{tab::filter+PCD}
\vskip 0.15in
\begin{center}
\begin{small}
\begin{sc}
\setlength{\tabcolsep}{1mm}
{\begin{tabular}{cccccc}
\hline
JacobiConv	&$\bf{0.0003}$&$\bf{0.0017}$&$\bf{0.0253}$&$\bf{0.0157}$&$\bf{0.2972}$\\\hline
Monomial	&$2.3257$	&$4.0960$	&$10.9556$	&$8.8068$	&$10.8799$\\
Chebyshev	&$1.0037$	&$2.3468$	&$7.8522$	&$6.0297$	&$9.1175$	\\
Bernstein	&$0.0110$	&$0.0058$	&$0.1517$	&$0.1607$	&$0.5705$	\\
\hline 
\end{tabular}}
\end{sc}
\end{small}
\end{center}
\vskip -0.1in
\end{table}

\section{Analysis Using Gradient Flow}\label{app::dynamics}
Using gradient flow method, i.e., gradient descent with infinitesimal steps~\citep{Optim}, we analyze the optimization of linear GNN. We prove that linear GNN can converge to the global minimum of loss function under mild conditions. 

Let $\sI$ denote the training set containing nodes. The prediction of a linear GNN is,
\begin{equation}
\begin{aligned}
Z_{il}=\sum_{k=0}^K\sum_{j=1}^d\alpha_{kl}(g(\hat L)X)_{ij}W_{jl}.
\end{aligned}
\end{equation}

Gradient flow method assumes that $\frac{\mathrm{d}}{\dd t}W_{jl}=-\frac{\partial L}{\partial W_{jl}}, \frac{\mathrm{d}}{\dd t}\alpha_{kl}=-\frac{\partial L}{\partial \alpha_{kl}}$.

Therefore,
\begin{align}
\frac{\mathrm{d}}{\mathrm{d}t}L
=\sum_{\text{all elements}}\frac{\dd L}{\dd Z}_{\sI\sI}\odot \frac{\dd Z_{\sI\sI}}{\dd t}=-\sum_{i\in\sI}\sum_{l=1}^{d'}\frac{\partial L}{\partial Z_{il}}(\sum_{k=0}^K\frac{\partial Z_{il}}{\partial \alpha_{kl}}\frac{\partial L}{\partial \alpha_{kl}}+\sum_{j=1}^d \frac{\partial Z_{il}}{\partial W_{jl}}\frac{\partial L}{\partial W_{jl}}).\\
\end{align}
\begin{equation}
\begin{aligned}
\frac{\partial L}{\partial \alpha_{kl}}=\sum_{i\in\sI}\frac{\partial L}{\partial Z_{il}}\frac{\partial Z_{il}}{\partial \alpha_{kl}}.
\end{aligned}
\end{equation}
\begin{equation}
\begin{aligned}
\frac{\partial L}{\partial W_{jl}}=
\sum_{i\in \sI}\frac{\partial L}{\partial Z_{il}}\frac{\partial Z_{il}}{\partial W_{jl}}.
\end{aligned}
\end{equation}
Therefore, 
\begin{equation}
\begin{aligned}
\frac{\mathrm{d}}{\mathrm{d}t}L
&=-\sum_{i\in\sI}\sum_{l=1}^{d'}\frac{\partial L}{\partial Z_{il}}
(\sum_{k=0}^K\frac{\partial Z_{il}}{\partial \alpha_{kl}}\sum_{i'\in\sI}\frac{\partial L}{\partial Z_{i'l}}\frac{\partial Z_{i'l}}{\partial \alpha_{kl}}+\sum_{j=1}^d \frac{\partial Z_{il}}{\partial W_{jl}}\sum_{i'\in \sI}\frac{\partial L}{\partial Z_{i'l}}\frac{\partial Z_{i'l}}{\partial W_{jl}})\\
&=-\sum_{l=1}^{d'}\sum_{i\in\sI, i'\in\sI}\frac{\partial L}{\partial Z_{il}}\frac{\partial L}{\partial Z_{i'l}} (\sum_{k=0}^K \frac{\partial Z_{il}}{\partial \alpha_{kl}}\frac{\partial Z_{i'l}}{\partial \alpha_{kl}}+\sum_{j=1}^d\frac{\partial Z_{il}}{\partial W_{jl}}\frac{\partial Z_{i'l}}{\partial W_{jl}})\\
&=-\sum_{l=1}^{d'}\frac{\partial L}{\partial Z_{\sI l}}^T
(M^{(l)}+S^{(l)})\frac{\partial L}{\partial Z_{\sI l}},\\
\end{aligned}
\end{equation}

where $M^{(l)}=\frac{\partial Z_{\sI l}}{\partial \alpha_{:l}}\frac{\partial Z_{\sI l}}{\partial \alpha_{:l}}^T$ and $S^{(l)}=\frac{\partial Z_{\sI l}}{\partial W_{:l}}\frac{\partial Z_{\sI l}}{\partial W_{:l}}^T$, where we define $\frac{\partial \vec a}{\partial \vec b}$, the derivative of vector $\vec a$ with respect to vector $\vec b$, is a matrix whose $i, j$ element is $\frac{\partial \vec a_i}{\partial \vec b_j}$.

Both $M^{(l)}$ and $S^{(l)}$ are symmetric semi-definite matrix. Let $\sigma_l$ denote the minimum eigenvalue of $M^{(l)}+S^{(l)}$, 

\begin{equation}
\begin{aligned}
\frac{\mathrm{d}}{\mathrm{d}t}L
\le -\sum_{l=1}^d \frac{\partial L}{\partial Z_{\sI l}}^T \sigma_l \frac{\partial L}{\partial Z_{\sI l}}.\\
\end{aligned}
\end{equation}

If $L$ is squared loss, namely $\frac{1}{2}\sum_{l=1}^{d'}\sum_{n\in\sI} ||Z_{nl}-Y_{nl}||_F^2$, and $\sigma_l>0$, linear GNN can always converge to global minimum of loss function.
\begin{equation}
\begin{aligned}
\frac{\mathrm{d}}{\mathrm{d}t}L \le -\sum_{l} \sigma_l  (Z_{\sI l}-Y_{\sI l})^T(Z_{\sI l}-Y_{\sI l}) \le -2\sigma_{\min} L,
\end{aligned}
\end{equation}
where $\sigma_{\min}=\min_l \sigma_l$. Let $L^*$ denote the minimum loss. $L^*>0$ and $\frac{{\dd L}^*}{\dd t}=0$.
\begin{equation}
\begin{aligned}
\frac{\mathrm{d}}{\mathrm{d}t}(L-L^*)\le -2\sigma_{\min} (L-L^*).
\end{aligned}
\end{equation}
Assuming $L_t$ denotes the loss at time t,
\begin{equation}
\begin{aligned}
(L_t-L^*)\le e^{-2\sigma_{\min} t} (L_0-L^*).
\end{aligned}
\end{equation}

Therefore, we have a prior guarantee of linear convergence to a global minimum for any graph with $\sigma_{\min}>0$. For any desired $\epsilon>0$, we have that $L_0-L^*<\epsilon$ for any T such that
\begin{equation}
\begin{aligned}
T\ge \frac{1}{2\sigma_{\min}}\log{\frac{L_{0}-L^*}{\epsilon}}.
\end{aligned}
\end{equation}

\section{Bases Visualization}\label{app::basesplot}

We visualize Monomial, Chebyshev, Bernstein and Jacobi Bases in Figure~\ref{fig::basesplot}. Note that Jacobi Bases change with hyperparameters. We use the hyperparameters for our image dataset here.
\begin{figure}[ht]
\begin{center}
\centerline{\includegraphics[scale=0.50]{ 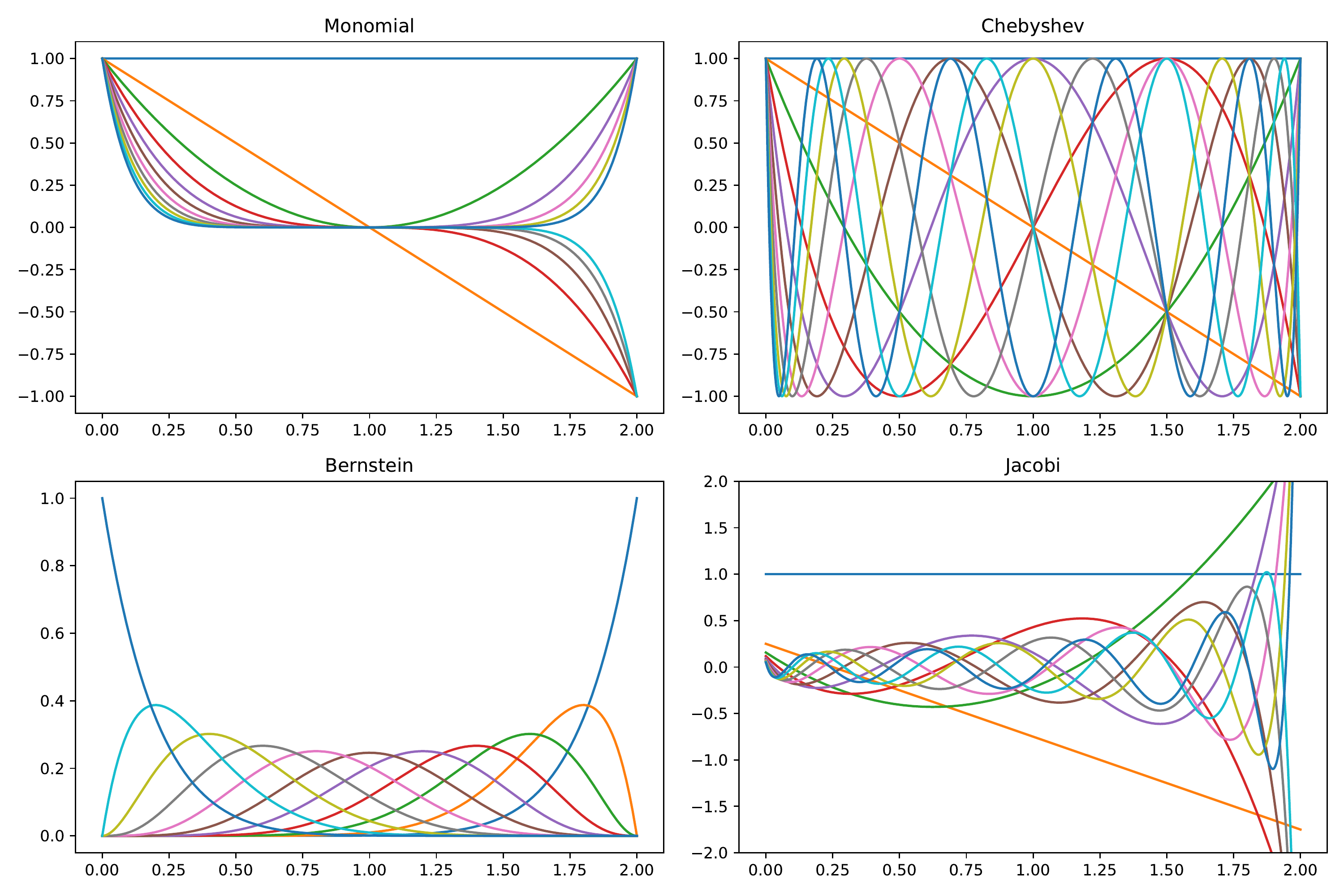}}
\caption{Polynomial bases visualization.}\label{fig::basesplot}
\end{center}
\vskip -0.2in
\end{figure}

\begin{figure}[ht]
\vskip -0.2in
\begin{center}
\centerline{\includegraphics[width=1.0\textwidth]{ 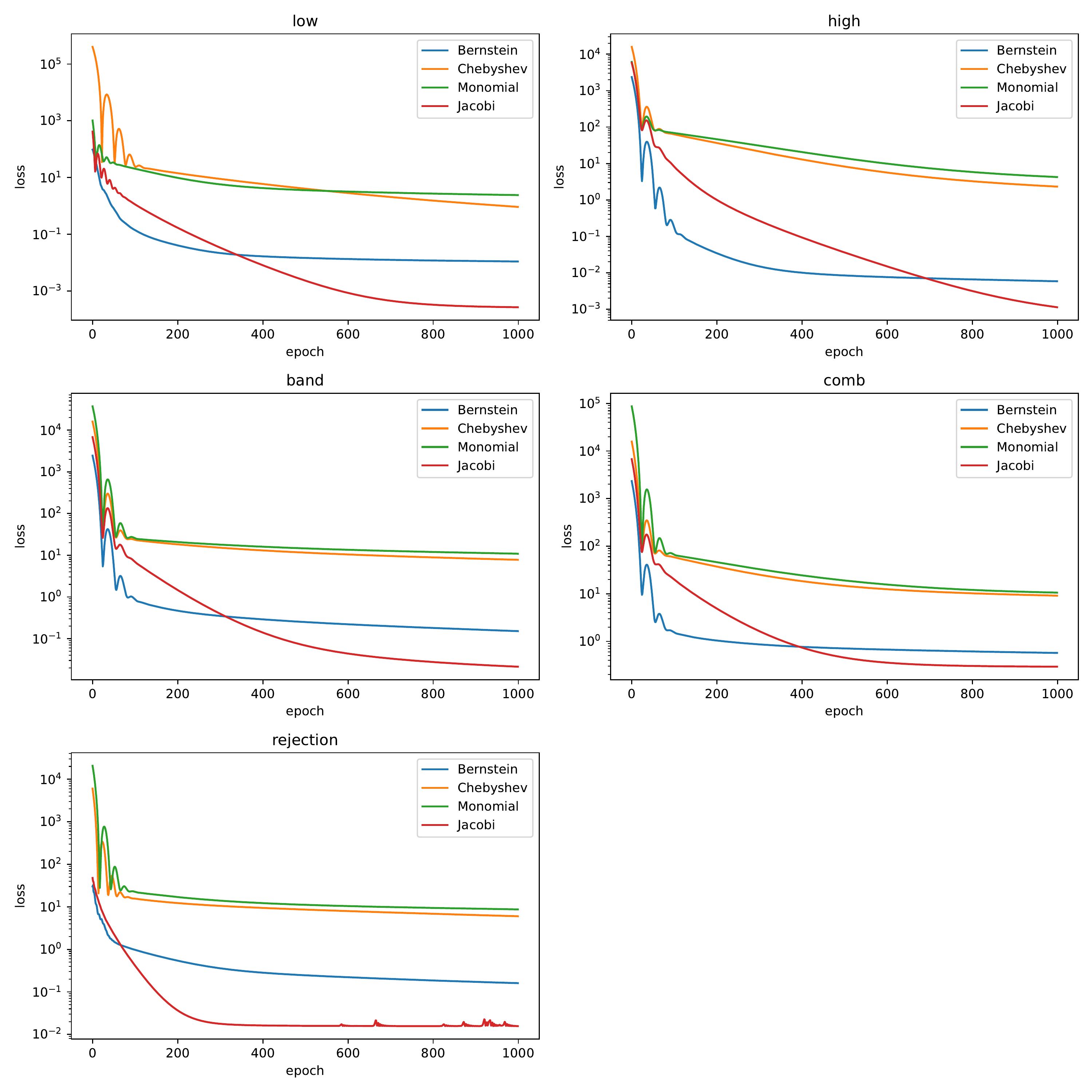}}
\caption{Training curve on some datasets.}\label{fig::ImgOptim}
\vskip -0.2in
\end{center}
\end{figure}

\begin{table}[htb]
\centering
\caption{Comparison between FullCoef and JacobiConv.}\label{tab::FullCoef}
\vskip 0.15in
\begin{center}
\begin{small}
\begin{sc}
\setlength{\tabcolsep}{0.5mm}
{
\begin{tabular}{lccccc}
\hline
&Cora &Citeseer &Pubmed &Computers &Photo
\\
\hline
FullCoef    &$87.37_{\pm 1.20}$ &$\bf{80.95_{\pm 1.06}}$ &$89.37_{\pm 0.64}$ &$89.94_{\pm 0.46}$ &$94.93_{\pm 0.50}$ \\
JacobiConv &$\bf{88.98_{\pm 0.46}}$ &$80.78_{\pm 0.79}$ &$\bf{89.62_{\pm 0.41}}$ &$\bf{89.96_{\pm 0.29}}$ &$\bf{95.43_{\pm 0.23}}$ \\
\hline
&Chameleon &Actor &Squirrel &Texas &Cornell\\
\hline
FullCoef &$64.79_{\pm 1.42}$ &$38.99_{\pm 1.37}$ &$49.61_{\pm 1.22}$ &$92.30_{\pm 4.52}$ &$91.48_{\pm 5.62}$\\
JacobiConv &$\bf{74.20_{\pm 1.03}}$ &$\bf{41.17_{\pm 0.64}}$ &$\bf{57.38_{\pm 1.25}}$ &$\bf{93.44_{\pm 2.13}}$ &$\bf{92.95_{\pm 2.46}}$\\
\hline
\end{tabular}}
\end{sc}
\end{small}
\end{center}
\end{table}

\section{The Optimization of linear GNN with Different Polynomial Basis}\label{app::ImgOptim}

In this section, we show how loss drops with different polynomial filter basis for linear GNNs in Figure~\ref{fig::ImgOptim}.

On all five datasets, Jacobi basis achieves the lowest loss and a higher convergence rate than Monomial and Chebyshev basis. However, Bernstein polynomial basis shows a high optimization rate in a few first epochs, which may attribute to that the parameter is far from the local minimum initially, and our approximation fails. In contrast, after a few epochs, Jacobi basis approaches the local minimum and shows a higher convergence rate.

\section{FullCoef vs JacobiConv}

JacobiConv first linear transforms node features and then filters the signal in each output dimension individually. In contrast, some models like ChebyConv~\citep{ChebyConv} use individual filter functions for each input dimension-output dimension pair to filter the signal in the input dimension and accumulate the filtered signal in the output dimension, which we call FullCoef. Though FullCoef may boost expressive power, extra parameters can also worsen the generalization. We compare FullCoef JacobiConv and original JacobiConv in Table~\ref{tab::FullCoef}. Results show that JacobiConv outperforms FullCoef on $9$ out of $10$ datasets. The extra express power that FullCoef brings is minor.

\end{document}